\documentclass[10pt,twocolumn,letterpaper]{article}

\usepackage[pagenumbers]{cvpr} 

\definecolor{cvprblue}{rgb}{0.21,0.49,0.74}
\usepackage[pagebackref,breaklinks,colorlinks,allcolors=cvprblue]{hyperref}
\usepackage{cite}
\usepackage{amsmath,amssymb,amsfonts, amsthm}
\usepackage{algorithm}
\usepackage{algorithmicx}
\usepackage[noend]{algpseudocode}
\usepackage{graphicx}
\usepackage{textcomp}
\usepackage{bbm}
\usepackage{cases}
\usepackage{multirow}
\usepackage{booktabs}
\usepackage{float}
\usepackage{comment}
\usepackage{pifont}
\usepackage{float}
\usepackage{needspace}

\newcommand{\R}{\mathbbm{R}}

\newcommand{\segu}{S}

\newcommand{\vox}{\mathbf{x}}
\newcommand{\voy}{\mathbf{y}}
\newcommand{\vod}{\mathbf{d}}

\newtheorem{thm}{Theorem}

\newtheorem{lem}{Lemma}
\theoremstyle{definition}

\def\paperID{Anonymous} 
\def\confName{CVPR}
\def\confYear{2026}

\title{D-Convexity: A Unified Differentiable Convex Shape Prior via Quasi-Concavity for Data-driven Image Segmentation}

\author{Shengzhe Chen, Hao Yan\\
School of Computing and Augmented Intelligence, Arizona State University\\
{\tt\small \{schen415, haoyan\}@asu.edu }
}

\begin{document}
\maketitle
\begin{abstract}
Convexity is a fundamental geometric prior that underlies many natural and man-made structures, yet remains challenging to impose effectively in end-to-end trainable segmentation networks. We revisit convexity from a functional perspective and propose a unified, threshold-free convexity prior based on the quasi-concavity of the network’s output mask function $u$. Instead of constraining a single binary segmentation, we require all super-level sets of \(u\) to be convex, transforming global shape constraints into local, differentiable inequalities on \(u\) and its derivatives. From this principle, we derive zero, first, and second-order characterizations, yielding respectively a local midpoint convexification algorithm, a gradient-based condition linked to supporting hyperplanes, and a sufficient second-order inequality expressed as a quadratic form on the tangent plane. The first and second-order formulations produce a compact convolutional loss that can be densely applied across the image without thresholding.  Our quasi-concavity losses integrate seamlessly with modern segmentation networks via the proposed convex gradient projection module (CGPM). They consistently enforce convexity and improve shape regularity across multiple datasets, outperforming networks tailored for retinal segmentation and surpassing previous shape-aware methods. Remarkably, our analysis unifies a wide spectrum of previous convex shape models, from discrete 1-0-1 line constraints and graph-cuts convexity formulations to curvature or signed distance Laplacian based level-set priors, within a single continuous and differentiable framework. Code is available at \tt{https://github.com/ShengzheC/D-Convexity}.
\end{abstract}
    
\section{Introduction}
\label{sec:intro}

\begin{figure*}[htbp]
    \centering
    \includegraphics[width=1\linewidth]{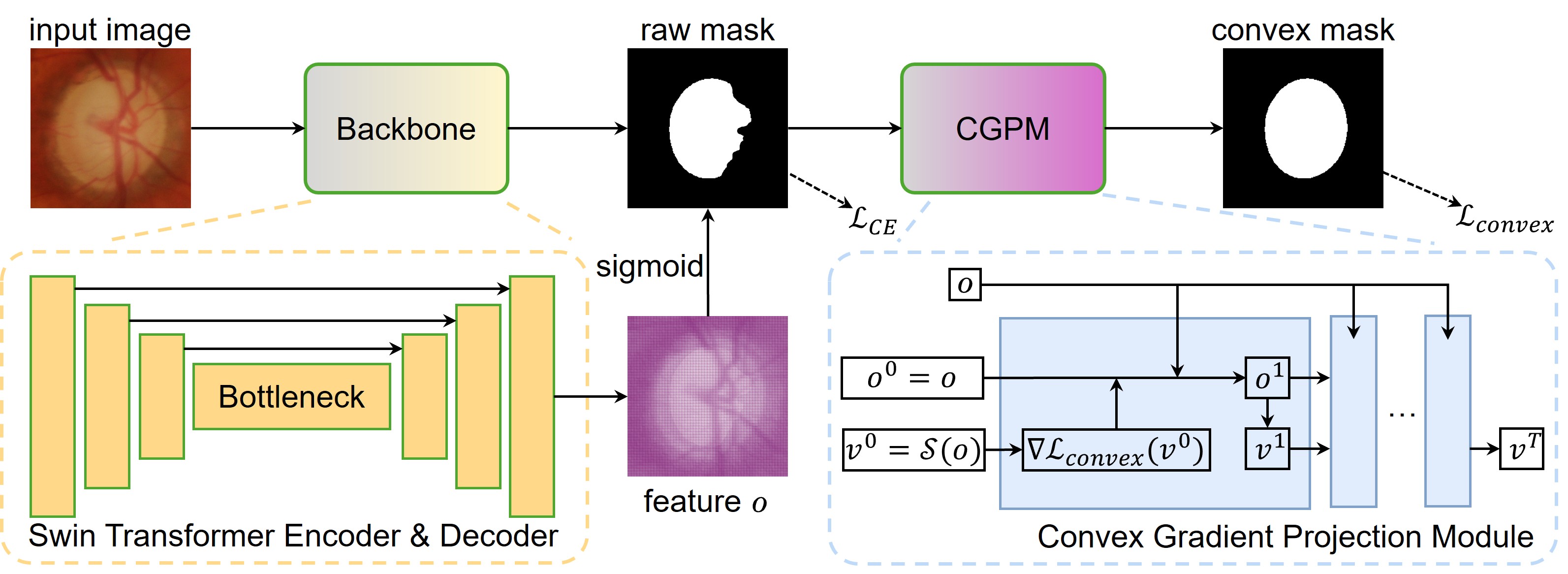}
    \caption{Architecture of the proposed framework.}
    \label{fig:flow}
\end{figure*}

Reliable image segmentation often depends on incorporating structural priors \citep{leventon2002statistical} that exclude implausible shapes, especially when data are noisy, limited, or occluded. Among them, convexity is a particularly powerful prior. Many anatomical and man‑made objects are convex or close to convex \citep{royer2016convexity}, and enforcing convexity helps segment touching instances while suppressing holes and irregular artifacts. Classical formulations demonstrate clear benefits, ranging from discrete energy models with convexity constraints or higher order potentials to continuous level set methods that regularize curvature or signed distance representations, yet these approaches remain difficult to integrate efficiently into modern neural architectures.

Existing convexity priors have several limitations. Discrete formulations \citep{royer2016convexity, gorelick2017convexity} encode convexity via global combinatorial constraints, e.g., 1–0–1 penalties on collinear triples. While effective, they rely on approximate solvers or expensive branch‑and‑cut procedures, and they are not straightforward to differentiate through or tightly coupled with deep networks.
Continuous level‑set methods \citep{yan2020convexity, luo2022convex} convert convexity into PDE constraints, e.g., requiring non‑negative curvature or enforcing the signed distance Laplacian to be non‑negative, yielding elegant theory. Yet such conditions are either necessary but not sufficient, or depend on a particular threshold of the segmentation function, which complicates their use as differentiable losses in probabilistic decoders.
Recent deep models \citep{liu2022deep, zhao2025convex, chen2025contour} incorporate shape priors into the loss, but they still lack explicit control over convexity.

Our idea is to introduce threshold‑free convexity via quasi‑concavity \citep{arrow1961quasi}.
We advocate a simple and functional view: instead of constraining a single binary mask, we constrain the entire segmentation function \(u:\Omega\!\to\![0,1]\). Specifically, we require \(u\) to be \emph{quasi‑concave}, i.e., all of its super‑level sets are convex. This threshold‑free prior turns a hard, set‑valued constraint on the segmented object \( S_{\gamma} = \{\vox\in\Omega | u(\vox)\ge\gamma\}\) into standard inequalities on \(u\) and its derivatives. We show:
(1) a zero‑order characterization that generalizes the classic “line segment stays inside” property to continuous masks,
(2) an equivalent first‑order condition expressed with image‑domain gradients,
and (3) a second‑order sufficient condition for the strict negativity of a quadratic form along the tangent direction of level sets. The second-order form leads to a local, differentiable penalty that can be computed efficiently within modern deep learning frameworks and applied densely without relying on any threshold. This penalty is then integrated into the neural network through a convex gradient projection module (CGPM).

Remarkably, we show that this work connects to and unifies many previous approaches.
Our analysis recovers several existing convex priors~\citep{han2020noise, liu2020convex, luo2022new, luo2023binary, ukwatta2013efficient, yang2017level, luo2019convex, yan2020convexity} as special cases. Certain discrete convexity checks on local windows can be derived from the first‑order condition. Curvature-based level‑set priors for signed distance functions follow from the second‑order condition, and graph-cuts based convexity can be viewed as enforcing zero‑order convexity on all pairs at the pixel graph scale. We make these links explicit and precise, thereby offering a single umbrella that clarifies when and why previous approaches succeed or fail.

Our contributions are summarized as follows:
\begin{itemize}
  \item \textbf{Quasi‑concavity as a unified convex prior.} We formalize the convexity of all super-level sets as the quasi‑concavity of the network output mask function $u$, yielding threshold‑free, image domain constraints suitable for modern deep learning framework. Our framework unifies and clarifies connections to many prior works on convexity-constrained segmentation.
  \item \textbf{Multi‑order characterizations and implementations.} We provide zero, first, and second‑order characterizations corresponding to different levels of spatial smoothness of the mask. We derive a local midpoint‑convexification algorithm for zero‑order enforcement and differentiable first/second‑order penalties that integrate seamlessly with modern segmentation frameworks. Effects of the proposed priors are validated on multiple datasets.
\end{itemize}

\section{Related Work}
\label{sec:formatting}

\noindent\textbf{Convexity in discrete optimization.}
Early shape priors leveraged graph‑cut friendly constraints such as the star shape prior \citep{veksler2008star}, later extended to geodesic star shape \citep{gulshan2010geodesic}, enabling multiple centers and path-aware constraints with global optimality guarantees in interactive settings.
A more general “hedgehog” prior \citep{isack2016hedgehog} constrains surface normals relative to a vector field for multi‑object segmentation, relaxing the rigidity of star‑convexity while retaining tractability.
Convexity proper was encoded via higher‑order clique penalties that forbid 1–0–1 labelings along straight lines \citep{gorelick2014convexity}, or as linear constraints added to multicut/ILP image decomposition, enabling multiple convex instances across classes \citep{royer2016convexity}.
These approaches establish strong baselines, but depend on discrete solvers and are not easy to differentiate through. 

\noindent\textbf{Continuous formulations with curvature and signed distance functions.}
A parallel line of work uses level sets to represent regions and enforces convexity through curvature based inequalities or properties of the signed distance function (SDF). Representative examples include convex segmentation via non‑negative curvature \citep{ukwatta2013efficient, yang2017level}, and explicit SDF constraints such as \(\Vert\nabla\phi\Vert{=}1\) with \(\Delta\phi\!\ge\!0\), which guarantee that sublevel sets are convex \citep{yan2020convexity}.
For multi‑object segmentation, \citep{luo2019convex} shows that a single level-set with a Laplacian sign constraint on a \(c\)-sublevel band can encode multiple disjoint convex objects efficiently and solve it with ADMM.
These continuous priors typically certify convexity at one chosen threshold (e.g., \(\phi{=}0\)) rather than across all output confidence levels. They also do not integrate naturally into deep learning based frameworks. 

\noindent\textbf{Convex priors in deep segmentation.} Recent works bake shape priors directly into neural losses. \citep{liu2020convex} introduces a convex shape block for DCNNs that projects sigmoid activations onto a convex shape set using a soft thresholding dynamics, enabling near binary convex outputs and inclusion constraints.
\citep{han2020noise} proposes a convex shape prior term for pupil segmentation, combining pixel-wise cross‑entropy with a loss that encourages convexity of the predicted region. \citep{pal2024convex} proposes a framework that guarantees the convexity of the segmented output object by leveraging fundamental geometrical insights into the boundaries of convex-shaped objects. In contrast, our quasi‑concavity view provides (1) a threshold‑free convexity prior on the full mask function; (2) equivalent first‑order and sufficient second‑order conditions that lead to simple, differentiable losses; and (3) a theoretical bridge that recovers discrete convexity checks and SDF/curvature constraints as special cases, while being easy to deploy in modern neural networks and scalable training.

\begin{figure}[htbp]
    \centering
    \includegraphics[width=1\linewidth]{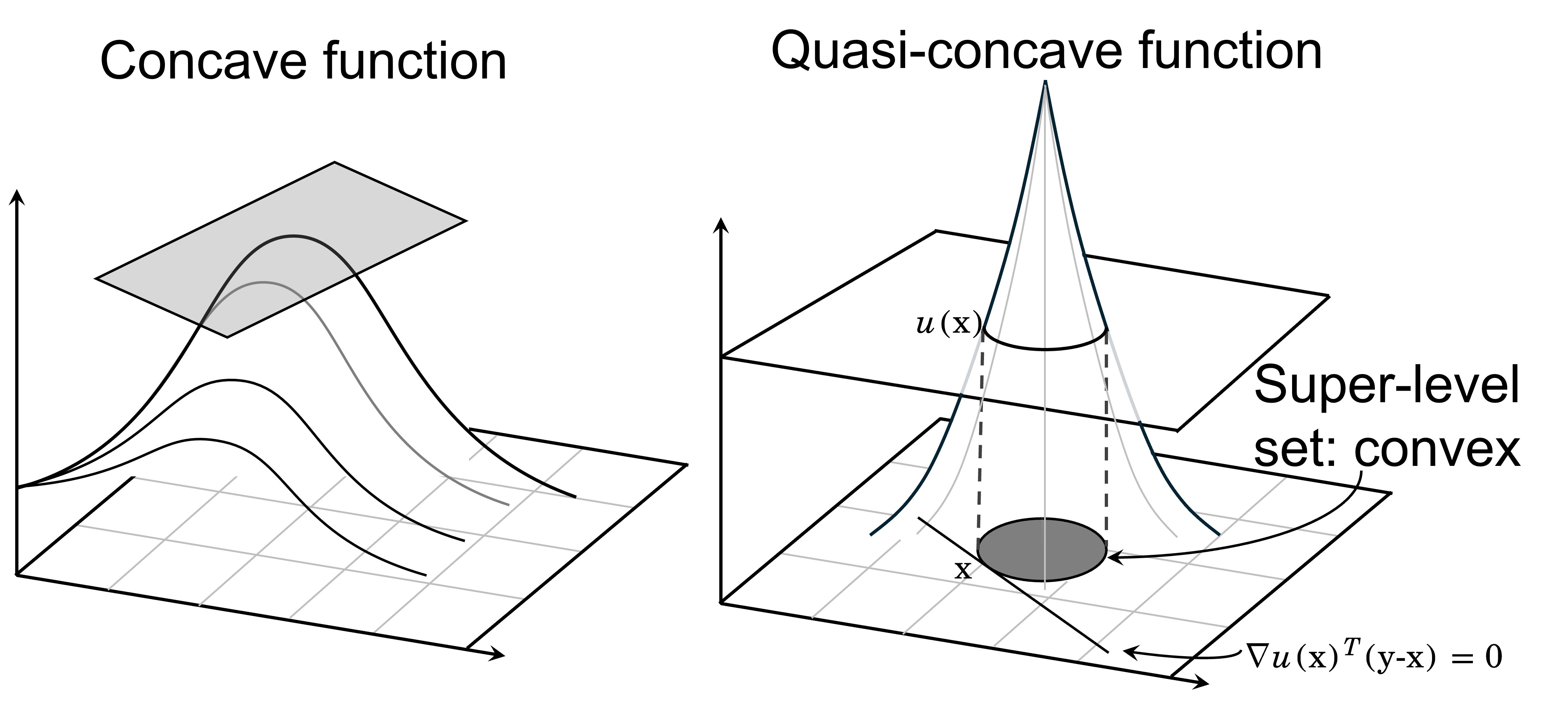}
    \caption{Concave vs. quasi-concave functions. A concave function lies below its tangent plane, whereas a quasi-concave function satisfies a weaker condition while still retaining convex super-level sets.}
    \label{fig:quasi}
\end{figure}
\section{Method and Theory}

\subsection{Data-driven image segmentation}

Let $\Omega \subset\mathbbm{R}^d$ denote the pixel domain and $I$ the input image, which contains a foreground object to be segmented. For clarity, we first consider the binary segmentation problem, and the multi-class case will be discussed later. The ground truth mask is defined by  $g: \Omega\to\{0,1\}$.  The goal is to predict a mask function $u: \Omega \to [0,1]$ using a neural network $\mathcal{N}(I;\Theta)$, 
\begin{equation}
    \left\{\begin{aligned}
          \Theta^* &= \arg\min_{\Theta}(\mathcal{L}_{\text{fid}}(u,g) + \mathcal{L}_{\text{reg}}),  \\
          o &= \mathcal{N}(I;\Theta^*), \\
          u &= \mathcal{S}(o), 
    \end{aligned}\right.
\end{equation}
so that $u(\mathbf x)$ approximates the probability that pixel $\mathbf x\in\Omega$ belongs to the object $\{\mathbf x:g(\mathbf x) = 1\}$. $\mathcal{L}_{\text{fid}}$ is a fidelity loss, e.g. Dice or cross-entropy loss, and $\mathcal{L}_{\text{reg}}$ is a regularization loss term that imposes desired properties on $u$ or $\mathcal{N}$. $o$ is the segmentation feature extracted by a neural network, and $\mathcal{S}$ is a sigmoid or softmax function, to make $u$ a probability map whose values are in $[0,1]$. Consequently, by thresholding  $\gamma\in \R$, we obtain a binarized segmentation and the final \textbf{segmented object} $\segu_\gamma = \{ \vox\in\Omega| u(\vox)\ge \gamma \}$.

\subsection{Convex prior and quasi-concavity}

In many domain-specific tasks, especially in case of high noise levels, little training data, or occlusions, the assumption that the segmented object is convex can significantly improve the segmentation \citep{royer2016convexity}. This suggests constraining $\segu_\gamma$ to be a convex set. 

However, because the segmented object depends on the threshold, we actually require a stronger  convex prior which is threshold-free. To do so, notice $\segu_\gamma$ is a super-level set of $u$. If the  convexity is satisfied regardless of the threshold choice, it means $u$'s super-level sets are always convex, which is the definition of $u$ being \textbf{quasi-concave} \citep{arrow1961quasi}. Equivalently, this stronger version of convexity  yields the following convex prior:
\begin{align}
    \text{$u$ is quasi-concave} \Longleftrightarrow \forall \gamma, \segu_\gamma \text{ is a convex shape.} 
\end{align}

In this way, we convert a hard geometric shape constraint on the segmented shape $\segu_\gamma$ into a functional constraint on the segmentation function $u$. This transformation offers significant advantages. First, it eliminates the need to choose a specific threshold. Second, by considering different levels of spatial smoothness, namely $u\in C^0,C^1,C^2$, we can respectively formulate the zero-order, first-order, and second-order conditions that constrain $u$ to be quasi-concave. Third, directly imposing constraints on the output segmentation function $u$ is key to achieving a differentiable convex shape prior within a deep learning framework.

\subsection{Zero-order condition}

Suppose $u\in C^0$, meaning that it is continuous but not assumed to be differentiable in the image domain. In this case, the following theorem guarantees $u$ to be quasi-concave:

\begin{thm}[Zero-order quasi-concavity condition]
    $u\in C^0$ is quasi-concave $\Longleftrightarrow$ For any $\vox, \voy \in \Omega,$  $\lambda \in [0,1]$, $u(\lambda \vox + (1-\lambda) \voy ) \ge \min\{u(\vox), u(\voy)\}$.

    \begin{proof}[Proof sketch]

Every super-level set is convex means that any line segment joining two points above a given level must also remain above that level, yielding the inequality. Conversely, if the inequality holds for all pairs, 
each super-level set is closed under convex combinations and thus convex.  See the rigorous proof in the Supplementary. 
\end{proof}

    \label{thm:0}
\end{thm}

\subsection{First-order condition}

In practical segmentation scenarios, for example in the eye fundus segmentation, the mask $u$ is typically expected to exhibit some degree of spatial smoothness. In that case, assume $u\in C^1$. We start by the following lemma.

\begin{lem}[Supporting hyperplane given by the gradient]
\label{lem:1}
Let $u:\Omega \to \mathbb{R}$ be a $C^1$ function. Fix $\gamma \in \R$ and consider the super-level set $\segu_\gamma := \{\vox \in \Omega | u(\vox) \ge \gamma\}.$
Assume $\segu_\gamma$ is convex. Let $\voy \in \partial \segu_\gamma$ be a boundary point with $u(\voy)=\gamma$ and suppose $\nabla u(\voy)\neq 0$. Then the affine hyperplane, 
\begin{align}
    T_\voy := \{\vox \in \Omega | \nabla u(\voy)^\top (\vox - \voy) = 0\}
\end{align}
is a supporting hyperplane of $\segu_\gamma$ at $\voy$, i.e. $\segu_\gamma$ is contained in the closed half-space
\begin{align}
    \segu_\gamma\subset H_\voy:= \{\vox\in\Omega | \nabla u(\voy)^\top (\vox - \voy) \ge 0\}.
\end{align}

In particular, $T_\voy$ is also tangent to the contour $\partial \segu_\gamma = \{\vox\in\Omega | u(\vox) = \gamma  \}$, therefore, the normal vector of any supporting hyperplane at $\voy$ can be chosen parallel to $\nabla u(\voy)$.
\end{lem}

\begin{proof}[Proof sketch]
Convexity of each super-level set implies that moving from a boundary point into the set cannot decrease $u$. 
Hence the gradient at the boundary must point inward, defining a supporting half-space. 
The tangent plane orthogonal to $\nabla u(\voy)$ therefore supports the set at $\voy$ and coincides with the local tangent of the level contour. See the rigorous proof in the Supplementary. 
\end{proof}

In the case of $u\in C^1$, using its spatial gradient $\nabla u$ in the pixel domain, we can derive an equivalent first-order constraint for $u$ being quasi-concave.

\begin{thm}[First-order quasi-concavity condition]
    $u\in C^1$ is quasi-concave $\Longleftrightarrow$ If $u(\vox) \ge u(\voy)$, then $\nabla u(\voy)^\top (\vox - \voy) \ge 0$.

    \begin{proof}[Proof sketch]
If $u$ is quasi-concave, each super-level set is convex. By Lemma~\ref{lem:1}, its gradient at any boundary point defines an inward normal, yielding the stated inequality.  
Conversely, if the inequality holds, the gradient at every point defines a supporting direction for the local super-level set, preventing any decrease of $u$ along line segments within it.  
Hence all super-level sets are convex and $u$ is quasi-concave. See the rigorous proof in the Supplementary. 
\end{proof}
    \label{thm:1}
\end{thm}

\subsection{Second-order condition}

The most informative characterization arises when we assume $u \in C^2$. 
Although not all segmentation masks satisfy such smoothness, 
this assumption is frequently adopted in prior works. 
In this setting, we examine the following optimization problem associated with a \textit{fixed} point $\vox \in \Omega$, and assume $\nabla u(\vox) \neq 0$:
\begin{equation}
\begin{aligned}
    &\max_{\voy \in \Omega} \; u(\voy), \\
    &\text{s.t. } \nabla u(\vox)^\top (\voy - \vox) \le 0.
\end{aligned}
\label{eq:opt}
\end{equation}

This formulation seeks the point $\voy$ on the non-supporting side of the hyperplane 
$T_\vox$ that maximizes $u(\voy)$. 
If $u$ is quasi-concave, then by definition its super-level set $\segu_{u(\vox)} = \{ \voy \in \Omega \mid u(\voy) \ge u(\vox) \}$ is convex. The point $\vox$ therefore lies on the boundary of $\segu_{u(\vox)}$, and according to Lemma \ref{lem:1}, the hyperplane $T_\vox$ serves as tangent and supporting hyperplane to this set at $\vox$.
Consequently, all feasible points $\voy$ satisfying 
$\nabla u(\vox)^\top (\voy - \vox) \le 0$ lie on the non-supporting side, whereas all $\voy$ such that $u(\voy) \ge u(\vox)$ lie on the supporting side. This implies $\vox$ \textit{maximizes} problem~\eqref{eq:opt}.

Then, we know that the constraint in \eqref{eq:opt} is \emph{binding} (active) at the optimum 
(i.e., $\nabla u(\vox)^\top (\voy - \vox) = 0$), which provides additional insight into the local structure of $u$. Due to binding, the  feasible displacements $\vod$ around $\vox$ must satisfy the tangency constraint $\vod\in T_\vox = \{\, \vod | \nabla u(\vox)^\top \vod = 0 \,\}.$

For $u \in C^2$, the second-order Taylor expansion at $\vox$:
\begin{equation}
    \begin{aligned}
    u(\vox + \vod)
    = u(\vox)
    + \nabla u(\vox)^\top \vod + \frac{1}{2}\vod^\top \nabla^2 u(\vox) \vod
    + o(\|\vod\|^2).
\end{aligned}
\end{equation}

Since $\vod \in T_\vox$, the linear term vanishes, and the variation of $u$ to second order becomes
\begin{align}
    u(\vox + \vod) - u(\vox)
    = \frac{1}{2}\vod^\top \nabla^2 u(\vox) \vod
    + o(\|\vod\|^2).
\end{align}

Hence, a necessary condition for $\vox$ to be a local maximizer of $u$ under the constraint 
$\nabla u(\vox)^\top (\voy - \vox) \le 0$ is that this second-order variation is non-positive 
for all feasible directions $\vod\in T_\vox$, that is,
\begin{equation}
    \vod^\top \nabla^2 u(\vox) \vod \le 0,
    \quad \forall\, \vod \in T_\vox.
    \label{eq:hessian_tangent}
\end{equation}
Equivalently, the Hessian $\nabla^2 u(\vox)$ must be \emph{negative semi-definite on the tangent space} $T_\vox$. Collectively, we get the necessary condition as the following theorem:

\begin{thm}[Second-order quasi-concavity \textbf{necessary} condition]
$u \in C^2$ is quasi-concave $\Longrightarrow$ The Hessian $\nabla^2 u(\vox) \preceq 0$ (negative semi-definite) on tangent space $T_\vox$, for  $\vox\in\Omega$ such that $\nabla u(\vox)\neq 0$.
\end{thm}

However, this necessary condition cannot be used directly as a constraint on $u$, because it is not sufficient. The following theorem states the sufficient condition in the second-order case.

\begin{thm}[Second-order quasi-concavity \textbf{sufficient} condition]
Let $u \in C^2$. If the Hessian $\nabla^2 u(\vox) \prec 0$ (strict negative definite) on tangent space $T_\vox$ for all $\vox\in\Omega$ such that $\nabla u(\vox)\neq 0$, then $u$ is quasi-concave.

\begin{proof}[Proof sketch]
Along any line segment in a super-level set, define $\phi(t)=u((1-t)\vox_0+t \vox_1)$.  
If $\phi$ were to decrease below its endpoints, it would attain an interior minimum where $\phi'(t)=0$ and $\phi''(t)\ge0$.  
At such a point, the direction of the segment lies in the tangent space, where the assumed negative-definite Hessian forces $\phi''(t)<0$, a contradiction.  
Hence $\phi$ is quasi-concave along every line, so all super-level sets are convex and $u$ is quasi-concave. See the rigorous proof in the Supplementary.
\end{proof}
\end{thm}

For 2D images, the second-order condition has a simple explicit form in terms of the derivatives of $u$, which is convenient for implementation as a constraint. Consider $\vox$ with $\nabla u(\vox) = (u_x, u_y)(\vox) \neq 0$. In this case, its tangent space $T_\vox$ is a 1D line. For all vectors in $T_\vox$, they share a parallel direction $\vod = (-u_y,u_x)(\vox)$. Therefore, the condition that the Hessian $\nabla^2u(\vox)\prec 0$ on the tangent space $T_\vox$ is simply:
\begin{equation}
   \begin{aligned}
    Q_2(\vox) &:= \vod^\top \nabla^2u(\vox) \vod \\ 
   &= (-u_y,u_x)^\top \begin{pmatrix} u_{xx} & u_{xy} \\ u_{yx} & u_{yy} \end{pmatrix} (-u_y,u_x)(\vox) \\
   &=  (u^2_x u_{yy} - 2 u_x u_y u_{xy} + u^2_y u_{xx})(\vox)< 0.
\end{aligned} 
\end{equation}

\begin{thm}[2D case second-order quasi-concavity sufficient condition] Let $\Omega \subset \R^2, u\in C^2$. If for all $\vox\in\Omega$ with $\nabla u(\vox)\neq 0$, $Q_2(\vox) = (u^2_x u_{yy} - 2 u_x u_y u_{xy} + u^2_y u_{xx})(\vox)< 0$, then $u$ is quasi-concave.
\label{thm:3}
\end{thm}

\begin{figure}
    \centering
    \includegraphics[width=1\linewidth]{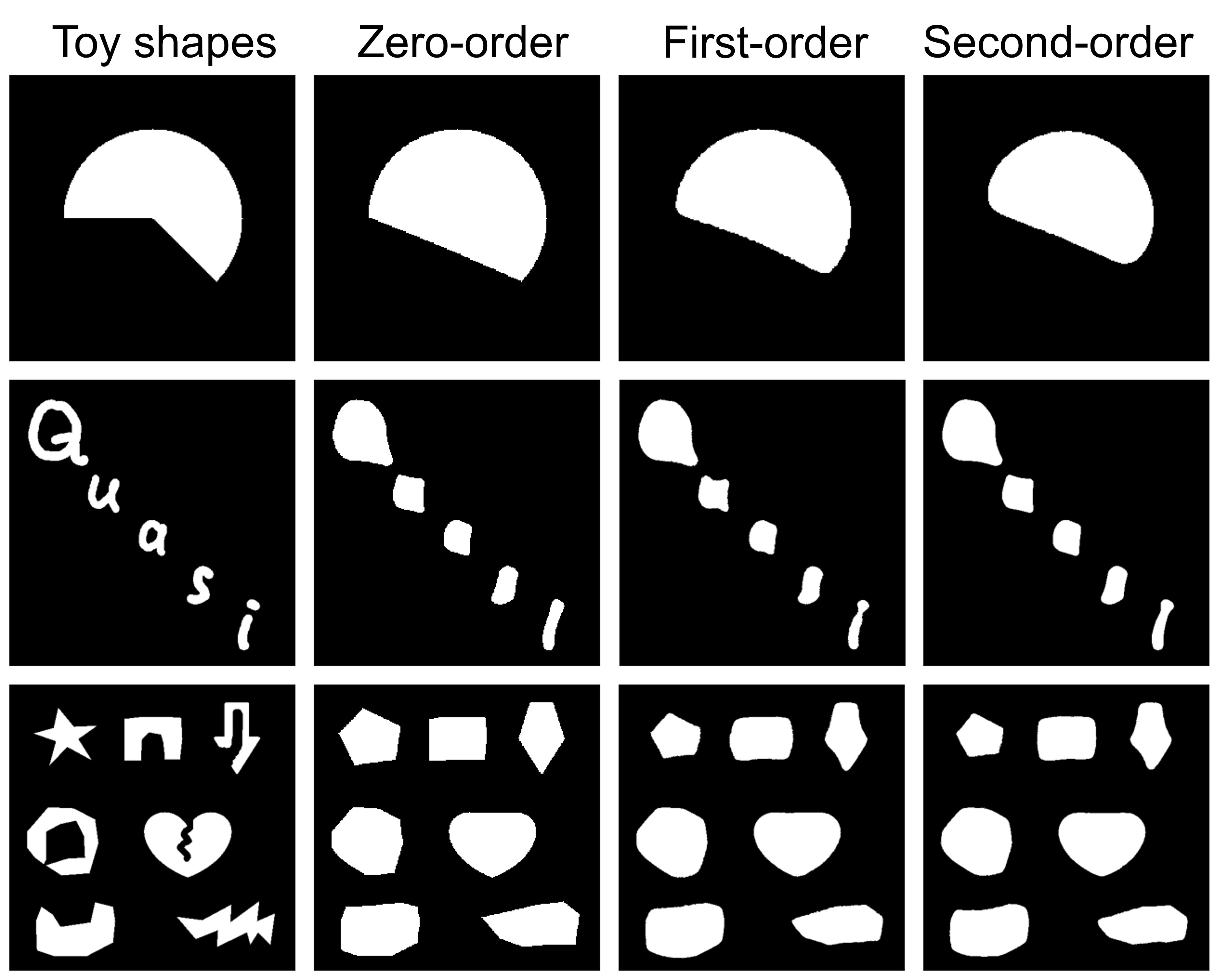}
    \caption{Toy illustration of convexification under the proposed priors. Several non-convex input shapes are regularized using our zero-, first-, and second-order convexity conditions. The resulting convexified shapes exhibit progressively stronger smoothness.}
    \label{fig:toy}
\end{figure}

\subsection{Unifying prior works on shape convexity}

We would like to emphasize that our convex shape prior by quasi-concavity actually unifies many prior works.

\citep{han2020noise} proposes a convex shape prior as ``for every $\vox,\voy$ in the segmentation object, their line segment also within it.'' This is captured by our zero-order condition (Theorem \ref{thm:0}). Given the condition, take $\vox, \voy\in \segu_\gamma$, the segmented object, then pixels on the line segment $\vox_\lambda = (1-\lambda)\vox + \lambda\voy$ should satisfy $u(\vox_\lambda) \ge \min\{ u(\vox), u(\voy) \}\ge\gamma$, which means $\vox_\lambda$ is also in the segmented object $\segu_\gamma$. This shows our zero-order condition is more general, because we actually require the condition over the entire image domain, which allows the prior to be applied to continuous mask functions.

\citep{liu2020convex, luo2022new, luo2023binary} propose a binary convexity characterization for an indicator mask $u:\Omega\to\{0,1\}$:
\begin{align}
    (u - 1)(b_r\ast(2u-1))\ge0, \ \forall r>0,
\end{align}
where $b_r$ is a function whose support is on a radius $r$ disk with $\int_{\Vert\vox\Vert\le r}b_r(\vox)d\vox = 1$. When $u(\vox) = 1$, this condition trivially holds. When $u(\vox) = 0$ (background), this yields $b_r \ast 1 \ge 2(b_r\ast u)$, i.e. $\frac{|B_r\cap \segu|}{|B_r|}\le \frac{1}{2}$, which means at any background pixel, the proportion of the foreground region within the neighborhood ball $B_r$ does not exceed one half. However, this condition can be derived directly from our first-order condition. Given any background pixel $\voy\in\Omega\backslash S$, by Lemma \ref{lem:1} and Theorem \ref{thm:1} we know the segmentation is fully contained in the half-space, $S\subset H_\voy$. Then, we have $B_r(\voy)\cap \segu\subset B_r(\voy)\cap H_\voy$. The latter is exactly a half-disk, which yields $|B_r(\voy)\cap \segu| \le \frac{1}{2} |B_r(\voy)|$.

\citep{ukwatta2013efficient, yang2017level} propose a curvature condition for convex shapes using a level-set function $\phi$ to represent segmented object, and constrain positive curvature $\kappa \ge0$, while \citep{luo2019convex, yan2020convexity} propose to use signed distance function, i.e. impose $\Vert \nabla \phi \Vert =1$ together with  $\kappa= \text{Div}(\frac{\nabla\phi}{\Vert \nabla\phi \Vert}) = \Delta\phi \ge0.$

Importantly, these works define level-set function that takes negative values inside the segmentation region, so the gradient $\nabla\phi$ points \textit{outward} the segmented object, while in our paper $u$ is defined to take higher value inside segmentation, so the gradient $\nabla u$ points \textit{inward} the segmented object. The two conventions differ only by sign. Let $\phi = -u$, the curvature condition:
\begin{equation}
\begin{aligned}
    \kappa(\vox)  = -\frac{u^2_x u_{yy} - 2 u_x u_y u_{xy} + u^2_y u_{xx}}{(u^2_x + u^2_y)^\frac{3}{2}}(\vox) = \frac{-Q_2(\vox)}{\Vert \nabla u(\vox) \Vert^3}.
\end{aligned}
\end{equation}

Since we constrain the quadratic form $Q_2(\vox)$ strictly negative whenever $\nabla u(\vox) \neq 0$, we naturally get $\kappa(\vox) >0$ as the convex condition. The $\kappa \ge0$ in previous works, which is equivalent to $Q_2(\vox) \le 0$, as  mentioned above, is a necessary but not sufficient condition for convexity.

\section{Integration into Neural Networks}

Theorem \ref{thm:0}, \ref{thm:1} and \ref{thm:3} state the conditions under which $u$ is quasi-concave for the cases $u\in C^0,C^1,C^2$ respectively. In this section, we derive practical approaches for integrating these conditions into deep segmentation frameworks. 

Remarkably, the second-order condition is superior to the other two in that it only requires checking all $\vox \in \Omega$ with spatial complexity $\mathcal{O}(|\Omega|)$, while the other two need to check all pairs $\vox,\voy \in \Omega$, making the spatial complexity $\mathcal{O}(|\Omega|^2)$. For implementation, $\mathcal{O}(|\Omega|^2)$ is prohibitively expensive. We therefore restrict the check to an $r$-radius window in the neighborhood of $\voy\in\Omega$, denoted as $N_\voy$, and only check $\vox \in N_\voy$. This reduces spatial complexity to $\mathcal{O}(r^2|\Omega|)$. In practice, since all local checks can be performed in parallel over the entire domain, they are efficient.

To make the proposed conditions computationally feasible, all continuous differential operators such as $u_x$, $u_y$, $u_{xx}$, $u_{xy}$, and $u_{yy}$ are replaced by discrete counterparts, which can be implemented efficiently as fixed convolution kernels corresponding to finite-difference operators~\citep{chen2025contour}.

\noindent\textbf{Zero-order midpoint convexification.}
Quasi-concavity requires \(u(\mathbf{m}) \ge \min\{u(\vox),u(\voy)\}\) for \(\vox,\voy\in\Omega,\mathbf{m}\in[\vox,\voy]\), which is impractical to enforce globally. We adopt a local, discrete midpoint rule: for each pixel \(\voy\) and offset \(\vod\) with \(\|\vod\|\le r\), let \(\mathbf{m}=\voy+\vod\) and \(\mathbf{z}=\voy+2\vod\) (the reflection of \(\voy\) across \(\mathbf{m}\)), and perform
$
u(\mathbf{m}) \leftarrow \max\!\big(u(\mathbf{m}),\, \min\{u(\voy),\,u(\mathbf{z})\}\big)
$
(Algorithm.~\ref{alg:local_convexify}). Iterating this update propagates convexity locally, increases \(u\) monotonically, and converges in finitely many steps. A large \(r\) approaches the convex hull of the foreground while a smaller \(r\) allows multiple objects to be convexified independently.

\begin{algorithm}[htbp]
\caption{Locally Midpoint Convexification}
\label{alg:local_convexify}
\begin{algorithmic}[1]
\Require continuous mask $u\in[0,1]^{H\times W}$, radius $r$, max iterations $T_{\max}$, terminal rate $\varepsilon$
\Ensure convexified mask $u$
\State $\mathcal{O}_r \gets \{\vod=(d_1,d_2)\in\mathbb{Z}^2 \mid 0<\sqrt{d_1^2+d_2^2}\le r\}$ \Comment{offset set within radius}
\State $\texttt{iter} \gets 0$
\While{$\texttt{iter} < T_{\max}$}
  \State $u_{\text{new}} \gets u$
  \ForAll{$\voy\in\Omega$ and $\vod \in \mathcal{O}_r$}
      \State $\mathbf{m} \gets$ $\voy + \vod$, \quad
      $\mathbf{z} \gets$ $\voy + 2\vod$
      \State
             $u_{\text{new}}(\mathbf{m}) \gets 
             \max(u(\mathbf{m}), \min(u(\voy), u(\mathbf{z})))$
  \EndFor
  \If{$\|u_{\text{new}} - u\|_\infty < \varepsilon$}
      \State \textbf{break} \Comment{converged}
  \Else
      \State $u \gets u_{\text{new}}$
  \EndIf
  \State $\texttt{iter} \gets \texttt{iter} + 1$
\EndWhile
\State \Return $u$
\end{algorithmic}
\end{algorithm}

 \noindent\textbf{First-order loss.} According to Theorem \ref{thm:1}, the pixels to be penalized are those satisfying $u(\vox) \ge u (\voy)$, but $\nabla u(\voy)^\top (\vox - \voy) < 0$. Accordingly, an indicator function $\mathbf{1}[u(\vox) \ge u (\voy)]$ should be used to select those $(\vox,\voy)$ pairs. For stable and smooth gradient propagation, we instead use a soft sigmoid function $(1+e^{-x/\varepsilon})^{-1}$ to approximate the hard indicator. This yields a penalty on the negative part of $\nabla u(\voy)^\top (\vox - \voy)$.  Consequently,
\begin{equation}
    \begin{aligned}
    \mathcal{L}_{1st}(u) = \frac{1}{|\Omega|} &\sum_{\vox\in N_\voy}\sum_{\voy\in \Omega} (   \text{Sigmoid}_{\varepsilon}(u(\vox) - u(\voy))\cdot  \\
    &  \text{ReLU}(-\nabla u(\voy)^\top (\vox - \voy)).
\end{aligned}
\end{equation}

 \noindent\textbf{Second-order loss.} By Theorem \ref{thm:3}, the $Q_2(\vox) < 0$ should hold for all $\vox \in \Omega$ such that $\nabla u(\vox) \neq 0$. If $\nabla u(\vox) = 0$, the primal optimization problem (\ref{eq:opt}) degenerates, and constraining the Hessian is no longer meaningful. Therefore, we use a $\Vert \nabla u(\vox) \Vert$ term to control this effect, i.e. when $\Vert \nabla u(\vox) \Vert =0$, no penalty is applied at that $\vox$. For $Q_2(\vox)$ being strictly negative, we introduce a margin $\delta>0$, and penalize the positive part of $Q_2(\vox) + \delta$. Collectively,
\begin{equation}
    \mathcal{L}_{2nd}(u) =\frac{1}{|\Omega|}\sum_{\vox\in \Omega} \Vert \nabla u(\vox) \Vert \cdot \text{ReLU}(Q_2(\vox) + \delta).
\end{equation}

\noindent\textbf{Convex gradient projection module.} Using the loss alone may not be sufficient to enforce a hard convexity constraint at inference time. To address this, we propose a convex gradient projection module (CGPM). After the network outputs the raw feature $o$, and raw mask function $u = \text{Sigmoid}(o)$, we solve the following optimization problem to find a proximal convexified mask function $u_p$:
\begin{align}
    u_p \in \arg\min_{v\in[0,1]} \frac{1}{2} \Vert v-u \Vert^2 + \lambda \cdot \mathcal{L}_{convex}(v),
\end{align}
where $\mathcal{L}_{convex}$ is one of $\mathcal{L}_{1st}, \mathcal{L}_{2nd}$. The problem can be solved by gradient descent, because $\nabla_v\mathcal{L}_{convex}(v)$ is available either through automatic differentiation or explicit derivation. To guarantee $v\in[0,1]$, the optimization is performed on the logit space $o$. CGPM can be viewed as an unrolled optimization module to solve the above problem. The detailed process is summarized in Algorithm~\ref{alg:cgpm}. Figure~\ref{fig:toy} shows convexification results on several toy shapes using the proposed priors. For the first and second conditions, we use a signed distance function passed through a sigmoid as the initial mask function, and simulate CGPM for several steps to minimize the proposed losses. The results show reliable convexification across all tested shapes.

\begin{algorithm}[htbp]
\caption{Convex Gradient Projection Module (CGPM)}
\label{alg:cgpm}
\begin{algorithmic}[1]
\Require Convex prior $\mathcal{L}_{convex}$, raw logits $o$, max iterations $T_{\max}$, learning rate $\eta$, regularizer coefficient $\lambda$
\Ensure Convexified mask $u_p$
\State $t \gets 0 \ , \ o^t \gets o, \  v^t \gets \text{Sigmoid}(o^t)$
\While{$t < T_{\max}$} 
  \State $o^{t} \gets o^{t} - \eta ((o^t - o) + \lambda \cdot \nabla \mathcal{L}_{convex}(v^t))$
  \State $v^{t} \gets \text{Sigmoid}(o^{t})$
  \State $t \gets t + 1$
\EndWhile
\State \Return $u_p = v^T$
\end{algorithmic}
\end{algorithm}

\noindent\textbf{Dealing with multi-class.}
With a proper choice of window size $r$, our convexity prior naturally supports multiple disjoint convex regions within a single foreground class, similar to \citep{luo2019convex},  
as shown in Figure~\ref{fig:toy}.
For $K$-class multi-label segmentation, we train $K$ independent binary maps $u_k:\Omega\!\to\![0,1]$, each with cross-entropy loss and optional convex prior. In softmax style single-label segmentation, let $u_0,\dots,u_{K-1}$ be the logits.  
The region of class $m$ is $\{u_m-\max_{i\ne m}u_i\ge0\}$, which is convex if $u_m-\max_{i\ne m} u_i$ satisfies the proposed quasi-concavity constraint.

\begin{table*}[htbp]
  \centering
  \caption{Performance of baseline and shape-aware methods on ACDC, CASIA, REFUGE and RIM-ONE-r3 datasets. Notably, models trained on REFUGE dataset are directly evaluated on RIM-ONE-r3 to assess generalization.}
  \label{tab:acdc-pupil-refuge}
  \resizebox{1\textwidth}{!}{
  \begin{tabular}{l *{12}{c}}
    \toprule
    \multirow{2}{*}{\textbf{Method}} &
    \multicolumn{3}{c}{\textbf{ACDC Dataset}} &
    \multicolumn{3}{c}{\textbf{CASIA Dataset}} &
    \multicolumn{3}{c}{\textbf{REFUGE Dataset}} &
    \multicolumn{3}{c}{\textbf{RIM-ONE-r3 Dataset}} \\
    \cmidrule(lr){2-4}\cmidrule(lr){5-7}\cmidrule(lr){8-10} \cmidrule(lr){11-13}
    & Dice $\uparrow$ & IoU $\uparrow$ & HD $\downarrow$ &
      Dice $\uparrow$ & IoU $\uparrow$ & HD $\downarrow$ &
      Dice $\uparrow$ & IoU $\uparrow$ & HD $\downarrow$ & Dice $\uparrow$ & IoU $\uparrow$ & HD $\downarrow$ \\
    \midrule
     U-Net~\citep{ronneberger2015u}  & 89.52 & 81.02 & 28.04 & 94.65 & 89.84 & 2.549 & 84.66 & 73.71 & 11.07 & 76.48 & 61.92 & 20.57 \\
     Swin-Unet~\citep{cao2022swin}  & 95.42 & 91.23 & 4.965 & 94.76 & 90.05 & 2.399 & 84.00 & 72.42 & 7.863 & 81.00 & 68.07 & 15.32 \\
    Dcan~\citep{chen2016dcan}  & 93.38 & 87.59 & 6.946 & 94.90 & 90.29 & 2.413 & 80.66 & 67.59 & 9.379 & 76.23 & 61.59 & 16.53 \\
    Dmtn~\citep{tan2018deep} & 92.60 & 86.22 & 8.500 & 94.92 & 90.34 & 2.337 & 82.36 & 70.01 & 9.337 & 78.39 & 64.46 & 16.80 \\
    ConvMCD~\citep{murugesan2019conv} & 93.44 & 87.68 & 15.53 & \textbf{95.03} & \textbf{90.54} & 2.323 & 78.38 & 64.45 & 12.51 & 76.71 & 62.22 & 18.18\\
    Active Boundary~\citep{wang2022active} & 90.93 & 81.38 & 24.71 & 94.49 & 89.55 & 2.656 & 84.82 & 73.63 & 10.59 & 75.37 & 60.48 & 20.64 \\
    \textbf{Proposed}    & \textbf{95.46}  & \textbf{91.31} & \textbf{4.702} & 94.71 & 89.94 & \textbf{2.288}  & \textbf{88.61} & \textbf{79.54} & \textbf{5.859} & \textbf{83.09} & \textbf{71.08} & \textbf{12.59} \\
    \bottomrule
  \end{tabular}
  }
\end{table*}

\section{Experiments}
\subsection{Experimental Settings}
\hspace*{1.2em}\textbf{Dataset.}
To evaluate the effect of the proposed convexity prior, we selected datasets that contain approximately convex anatomical structures. Specifically, we use two retinal fundus datasets, REFUGE~\citep{ORLANDO2020101570} and RIM-ONE-r3~\citep{rimone}, where the optic disc and cup regions exhibit convex or near-convex shapes. REFUGE provides a large set of labeled optic disc and cup images acquired under clinical conditions, with 400 training, 400 validation, and 400 test images. RIM-ONE-r3 consists of multi-expert annotations and 74 high-quality glaucomatous and 85 normal images. Following~\citep{Man16}, we use 99 images for training and 60 for testing. In addition to these retinal datasets, we also include ACDC~\citep{bernard2018deep} cardiac MRI slices in 2D short-axis views, which present convex ventricular structures in each slice, resulting in 100 training and 250 test images. We also evaluated on CASIA iris and pupil segmentation dataset annotated by~\citep{wang2020}, using 100 images for training and 99 for testing, and focusing on pupil segmentation.

\textbf{Pre-processing.}
For the retinal datasets, following common practice in retinal image segmentation, we first locate the optic disc center and crop a 480×480 region of interest around it to reduce background interference and emphasize the relevant anatomical structures. For general datasets, all images and masks are resized to a uniform resolution and normalized to [0,1]. Standard data augmentation is used during training, including random horizontal and vertical flips, rotations, scaling, Gaussian noise, Gaussian smoothing, and mild brightness and contrast perturbations. 

\textbf{Implementation Details.} Swin-Unet~\citep{cao2022swin} is used as the encoder-decoder backbone, trained on input images resized to 224x224. CGPM with $\mathcal{L}_{2nd}$ convex condition is then applied to the masks to enforce convexity, and the whole framework is trained with cross-entropy loss. We set parameters $\eta =1e-2, \lambda = 1, T_{\max}=100$. For training, we use AdamW optimizer and OneCycle scheduler for 200 epochs with batch size of 6 and maximum learning rate of $1e-4$. At inference, CGPM increases the per-image runtime of Swin-Unet from 0.01 seconds to 0.12 seconds.

\textbf{Evaluation Metrics.}
We evaluate segmentation performance using three standard metrics: Dice score, Intersection over Union (IoU), and Hausdorff distance (HD). Dice and IoU measure the overlap between the predicted and ground truth masks, reflecting region-level accuracy. Notably, instead of averaging per-image Dice and IoU, we compute these metrics by accumulating \textit{TP}, \textit{FP}, \textit{FN} over all images in each class. Hausdorff distance assesses the maximum boundary deviation, providing a complementary measure of contour precision and geometric consistency.

\subsection{Comparison with SOTAs}

We first compare our convexity regularized method  with several SOTA deep networks on the retinal dataset RIM-ONE-r3. Incorporating the knowledge that the optic cup is nested inside the optic disc, we train two continuous segmentation masks (disc and cup) and pass them through our CGPM to enforce convexity. The competing methods include recent retinal segmentation models specifically tailored to optic disc and cup delineation. Quantitative results in Table~\ref{tab:sota} show that our proposed method achieves higher Dice and IoU scores.

\begin{table}[htbp]
  \centering
  \caption{Comparison performance with SOTAs on RIM-ONE-r3.}
  \label{tab:sota}
  \resizebox{\columnwidth}{!}{
  \begin{tabular}{c c c c c c}
    \toprule
    \multirow{2}{*}{\textbf{Method}} & 
    \multirow{2}{*}{\textbf{Year}} &  
    \multicolumn{2}{c}{\textbf{Optic Disc}} & 
    \multicolumn{2}{c}{\textbf{Optic Cup}} \\
    \cmidrule(lr){3-4} \cmidrule(lr){5-6}
     &  & {Dice} & {IoU} & {Dice} & {IoU} \\
    \midrule
    U-Net~\citep{ronneberger2015u} & \textit{baseline} & 94.90 & 90.40 & 82.68  & 71.33 \\
    M-Net~\citep{fu2018joint} & \textit{2018} & 95.26 & 91.14 & 84.48  & 73.00\\
    CE-Net~\citep{gu2019net} & \textit{2018} & 95.27 & 91.15 & 84.35  & 74.24 \\
    Xu \textit{et al.}~\citep{xu2019mixed} & \textit{2019} & 95.61 & 91.72 & 85.64  & 75.87 \\
    CDED-Net~\citep{tabassum2020cded} & \textit{2020} & 95.82 & 91.01 & 86.22  & 75.32 \\
    GDCSeg-Net~\citep{zhu2021gdcseg} & \textit{2021} & 95.60 & 91.67 & 82.37  & 72.37 \\
    Liu \textit{et al.}~\citep{liu2022end} & \textit{2022} & 96.50 & 93.28 & 86.23  & 76.66 \\
    Swin-Unet~\citep{cao2022swin} & \textit{2022} & 96.77 & 93.75 & 86.57 & 76.37 \\
    Wang \textit{et al.}~\citep{wang2023towards} & \textit{2023} & 95.58 & 88.03 & 86.42  & 76.19 \\
    ODCS-NSNP~\citep{li2025odcs} & \textit{2025} & 96.63 & 93.48 & 85.41  & 75.04 \\
    \textbf{Proposed} & \textit{2026} & \textbf{97.15} & \textbf{94.46} & \textbf{88.01}  & \textbf{78.59} \\
    \bottomrule
  \end{tabular}
  }
\end{table}

\begin{figure*}[htbp]
    \centering
    \includegraphics[width=1\linewidth]{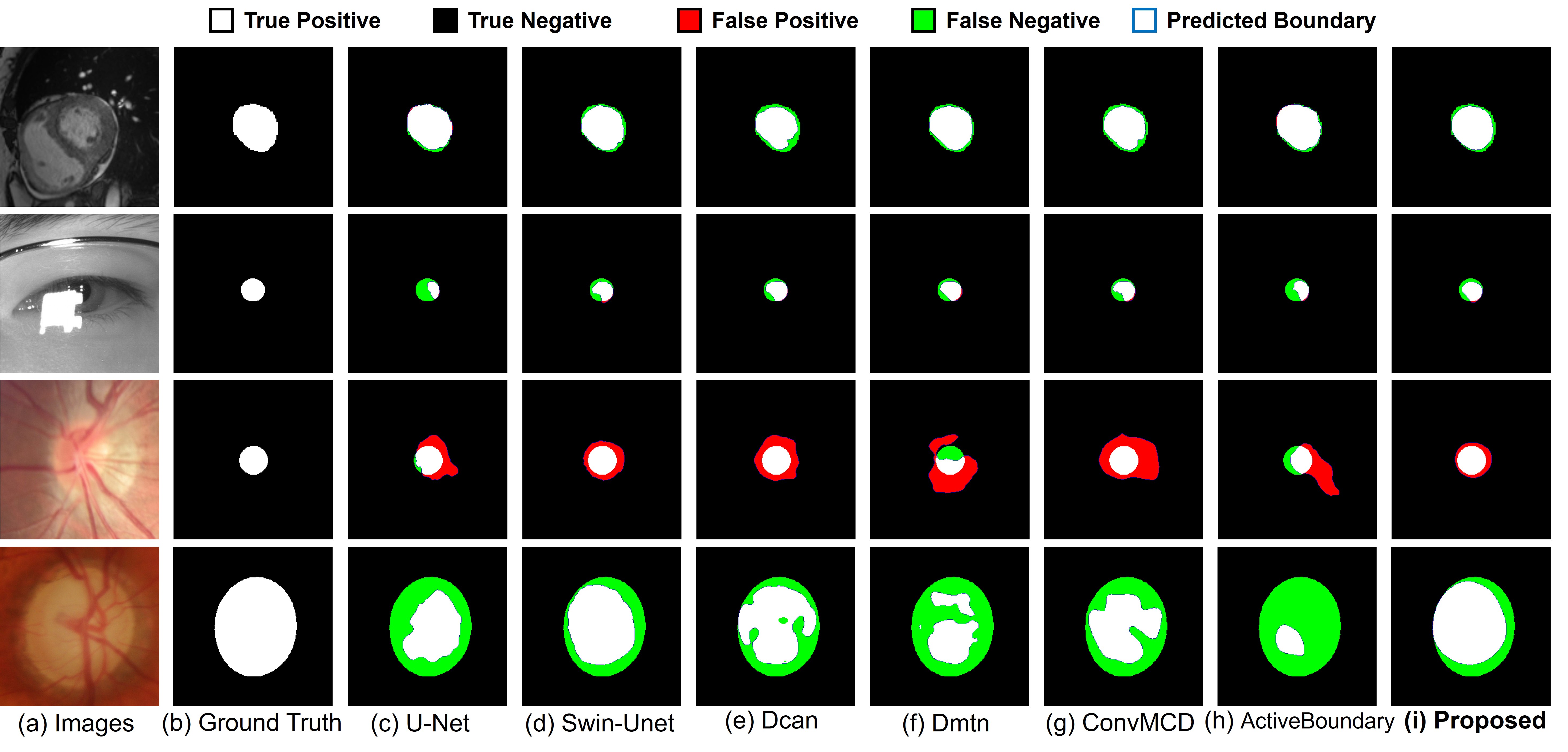}
    \caption{Qualitative visualization results comparison with shape-aware methods.}
    \label{fig:main}
\end{figure*}

\subsection{Comparison with Shape-Aware Methods}

We further compare our method with existing shape-aware segmentation approaches, including methods that use distance functions and boundary information to produce better-shaped segmentations. All methods are evaluated under consistent settings for fair comparison. For the two retinal datasets, we focus on optic cup segmentation. Models trained on REFUGE are directly applied to RIM-ONE-r3 to measure cross-dataset robustness. See the summarized quantitative results in Table~\ref{tab:acdc-pupil-refuge} and qualitative visualization in Figure~\ref{fig:main}. Overall, our convex prior achieves higher Dice and IoU, lower Hausdorff distance, and more regular boundaries, demonstrating its effectiveness as a generic and stable convex shape constraint across modalities.

\subsection{Ablation Study}

Throughout, we use Swin-Unet as the backbone and apply the second-order convex condition within CGPM. To evaluate design choices comprehensively, we perform ablation studies on individual components. For each setting, we run 10 paired trials on REFUGE and apply a t-test to assess the improvement in each metric and the corresponding p-values. Detailed results are available in the Supplementary.

\noindent\textbf{Effect of each order condition.}
We first examine the contribution of each order condition. The zero-order algorithm is used to convexify raw outputs, while the first and second conditions are applied within CGPM.  As summarized in Table~\ref{tab:ablation-backbone-delta-compact}, the second-order condition yields the best overall performance, achieving the largest gains in Dice and IoU while reducing the Hausdorff distance. 

\noindent\textbf{Backbone adaptability.}
To verify general applicability, we integrate the proposed convex prior into several baseline architectures, including U-Net~\citep{ronneberger2015u}, DeepLabV3p~\citep{chen2018encoder}, and Swin-Unet~\citep{cao2022swin}. The proposed prior consistently improves the boundary metrics (HD) of all models.  The first and second-order conditions significantly improve all metrics while enhancing shape convexity.

\begin{table}[htbp]
\centering
\caption{Ablation study on backbone choice and convexity conditions. Entries show improvement and $p$-values using t-test.}
\label{tab:ablation-backbone-delta-compact}
\resizebox{\columnwidth}{!}{
\begin{tabular}{c c cc cc cc}
\toprule
\multirow{2}{*}{\textbf{Backbone}} &
\multirow{2}{*}{\textbf{Prior}} & \multicolumn{2}{c}{\textbf{Dice} $\uparrow$} & \multicolumn{2}{c}{\textbf{IoU} $\uparrow$} & \multicolumn{2}{c}{\textbf{HD} $\downarrow$} \\ \cmidrule(lr){3-4}\cmidrule(lr){5-6}\cmidrule(lr){7-8}
 & & $\Delta$ & $p$ & $\Delta$ & $p$ & $\Delta$ & $p$ \\
\midrule
\multirow{3}{*}{U-Net~\citep{ronneberger2015u}} & $0th$  & $-$0.156     & .001     & $-$0.232   & $<$.001 & $-$0.098 & $<$.001 \\
& $1st$ & $+$0.429 & $<$.001  & 
$+$0.644  & $<$.001 & $-$0.832 & $<$.001 \\
      & $2nd$  & $+$1.510 & $<$.001  & $+$2.280 & $<$.001 & $-$2.240 & $<$.001 \\
\midrule
\multirow{3}{*}{Swin-Unet~\citep{cao2022swin}} & $0th$   & $-$0.185     & .005     & $-$0.284     & .004 & $-$0.056 &  .006 \\
& $1st$ &  $+$0.659 & $<$.001  & $+$1.013  & $<$.001 & $-$0.787 & .002 \\
      & $2nd$  & $+$2.375 & $<$.001  & $+$3.711  & $<$.001 & $-$1.630 & $<$.001 \\
\midrule
\multirow{3}{*}{DeepLab~\citep{chen2018encoder}} & $0th$  & $-$0.022     & .003     & $-$0.033     & .002 & $-$0.006 &  .003 \\
& $1st$ & $+$2.057 & $<$.001  & $+$2.772  & $<$.001 & $-$1.133 & $<$.001 \\
      & $2nd$  & $+$6.569 & $<$.001  & $+$9.202  & $<$.001 & $-$3.324 & $<$.001 \\
\bottomrule
\end{tabular}
}
\end{table}

\vspace{-1em}

\section{Conclusion}
We introduce a convex shape prior by enforcing quasi-concavity on the mask function, yielding a framework that explains and unifies many previous works. Our differentiable priors operate directly on continuous masks and integrates seamlessly into deep learning frameworks via the proposed convex gradient projection module (CGPM). Experiments and ablation studies show consistent metric gains while preserving shape convexity and regularity.
{
    \small
    \bibliographystyle{ieeenat_fullname}
    \bibliography{main}
}
\clearpage
\setcounter{page}{1}
\maketitlesupplementary

\appendix
\section{Proof of Theorem~\ref{thm:0}}
\textbf{Zero-order quasi-concavity condition:}
    $u\in C^0$ is quasi-concave $\Longleftrightarrow$ For any $\vox, \voy \in \Omega,$  $\lambda \in [0,1]$, $u(\lambda \vox + (1-\lambda) \voy ) \ge \min\{u(\vox), u(\voy)\}$.

    \begin{proof}
        ($\Rightarrow$)  Denote $\gamma = \min\{u(\vox), u(\voy)\}$. Apparently $u(\vox), u(\voy)\ge \gamma$, which means $\vox,\voy\in\segu_\gamma$. Given that $u$ is quasi-concave, its super-lever set $\segu_\gamma$ is convex.  Therefore, any convex combination of $\vox,\voy$, i.e. $\lambda \vox + (1-\lambda) \voy \in \segu_\gamma$. This means $u(\lambda \vox + (1-\lambda) \voy ) \ge \gamma = \min\{u(\vox), u(\voy)\}.$
        
        \medskip
        ($\Leftarrow$) For any $\gamma$, consider $\vox, \voy \in \segu_\gamma$. By definition, $u(\vox), u(\voy) \ge \gamma$. Therefore, $u(\lambda \vox + (1-\lambda) \voy) \ge \min\{u(\vox), u(\voy)\} \ge \gamma$. This means the convex combination $\lambda \vox + (1-\lambda) \voy \in \segu_\gamma$, i.e. $u$'s super-level set is convex, and $u$ is quasi-concave.
    \end{proof}

\section{Proof of Lemma~\ref{lem:1}}

\textbf{Supporting hyperplane given by the gradient:}
Let $u:\Omega \to \mathbb{R}$ be a $C^1$ function. Fix $\gamma \in \R$ and consider the super-level set $\segu_\gamma := \{\vox \in \Omega | u(\vox) \ge \gamma\}.$
Assume $\segu_\gamma$ is convex. Let $\voy \in \partial \segu_\gamma$ be a boundary point with $u(\voy)=\gamma$ and suppose $\nabla u(\voy)\neq 0$. Then the affine hyperplane, 
\begin{align}
    T_\voy := \{\vox \in \Omega | \nabla u(\voy)^\top (\vox - \voy) = 0\}
\end{align}
is a supporting hyperplane of $\segu_\gamma$ at $\voy$, i.e. $\segu_\gamma$ is contained in the closed half-space
\begin{align}
    \segu_\gamma\subset H_\voy:= \{\vox\in\Omega | \nabla u(\voy)^\top (\vox - \voy) \ge 0\}.
\end{align}

In particular, $T_\voy$ is also tangent to the contour $\partial \segu_\gamma = \{\vox\in\Omega | u(\vox) = \gamma  \}$, therefore the normal vector of any supporting hyperplane at $\voy$ can be chosen parallel to $\nabla u(\voy)$.

\begin{proof}
Since $u \in C^1$, the first-order Taylor expansion at $\voy$ gives, for $\vox$ near $\voy$,
\begin{align}
    \nu(\vox) := u(\vox) - u(\voy) = \nabla u(\voy)^\top (\vox - \voy) + r(\vox), \\ \qquad \text{where } \frac{r(\vox)}{\|\vox - \voy\|} \to 0 \text{ as } \vox \to \voy.
\end{align}

Take any $\vox \in \segu_\gamma$ (so $u(\vox) \ge u(\voy)$). Because $\segu_\gamma$ is convex and $\voy \in \partial \segu_\gamma$, the line segment
\begin{align}
    \ell(t) := \voy + t(\vox - \voy), \ t \in [0,1],
\end{align}

lies in $\segu_\gamma$. Define the scalar function $\phi(t) := u(\ell(t))$. Note that $\phi(0) = u(\voy)$ and $\phi(1) = u(\vox) \ge u(\voy)$, hence $\phi(t) \ge u(\voy)$ for all $t \in [0,1]$ because each point of the segment belongs to $\segu_\gamma$. Differentiating $\phi$ at $t = 0$ yields
\begin{align}
    \phi'(0) = \nabla u(\voy)^\top (\vox - \voy).
\end{align}

Since $\phi(t) \ge \phi(0)$ for small positive $t$, we must have $\phi'(0) \ge 0$, i.e.
\begin{align}
    \nabla u(\voy)^\top (\vox - \voy) \ge 0.
\end{align}

This inequality exactly says that the affine hyperplane $T_\voy$ is supporting, since every $\vox \in \segu_\gamma$ lies in the closed half-space determined by $T_\voy$.

Since $\nabla u(\voy)\neq 0$, the Implicit Function Theorem implies that the contour $\{\vox\in\Omega | u(\vox) = \gamma \}$ is a $C^1$ hypersurface in a neighborhood of $\voy$. Its
tangent space at $\voy$ is the kernel of $\nabla u(\voy)$, i.e.
\begin{equation}
    \begin{aligned}
    T_\voy &= \ker(\nabla u(\voy)) =\{\vod|\;\nabla u(\voy)^\top \vod = 0\} \\ 
    &=  \{\vox \in \Omega | \nabla u(\voy)^\top (\vox - \voy) = 0\}.
\end{aligned}
\end{equation}

If a supporting hyperplane at $\voy$ has normal $\mathbf{n}$, then by uniqueness of the tangent (since $u$ is differentiable and $\nabla u(\voy)\neq0$) we have $\mathbf{n}$ parallel to $\nabla u(\voy)$, which proves the final claim.
\end{proof}

\section{Proof of Theorem~\ref{thm:1}}

\begin{figure}[h]
    \centering
    \includegraphics[width=0.7\linewidth]{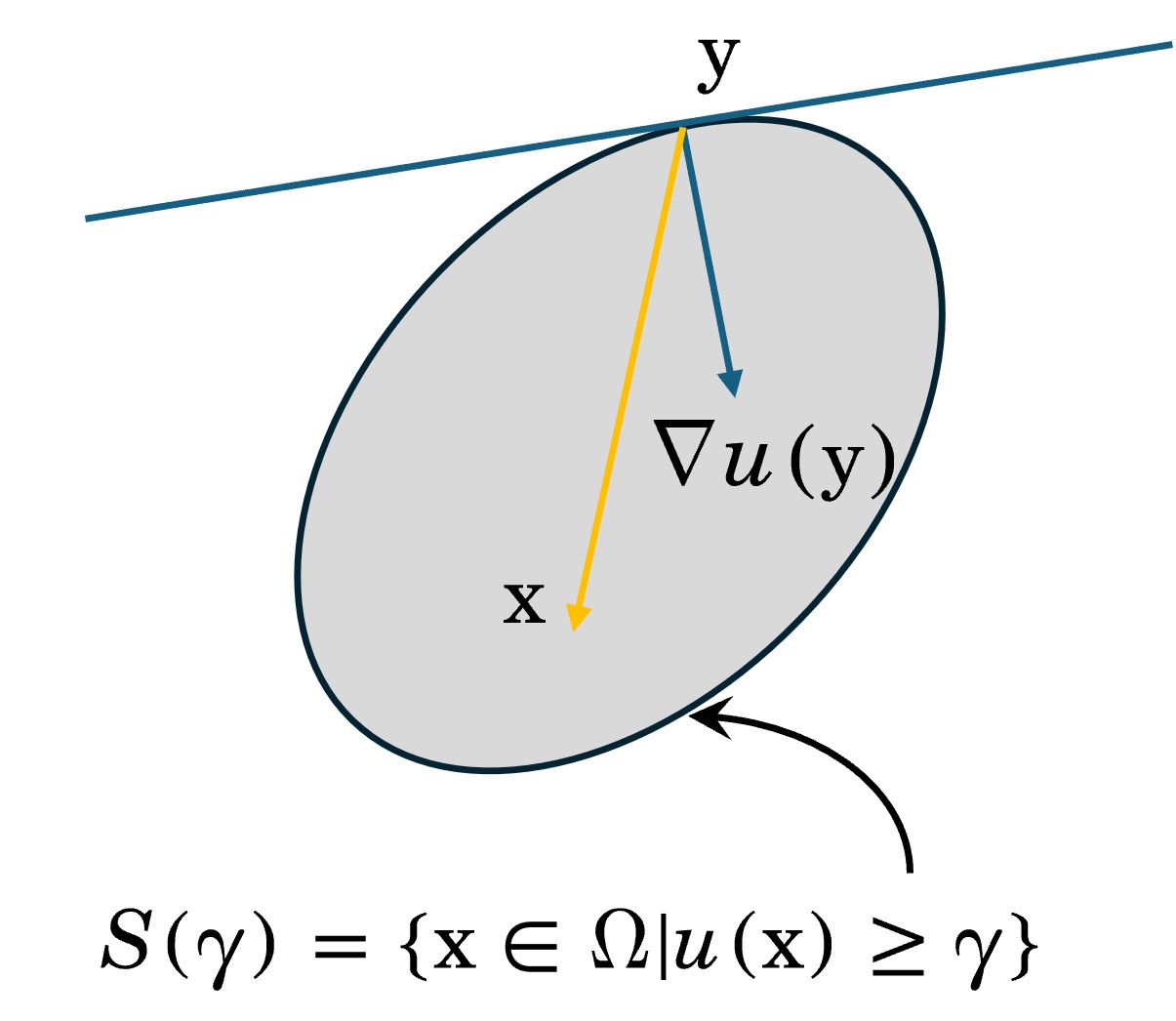}
    \caption{Illustration of the first-order condition. If $u(\vox)\ge u(\voy)$, then vectors $\vox- \voy$ and $\nabla u (\voy)$ must form an acute angle.}
    \label{fig:l1}
\end{figure}

\noindent\textbf{First-order quasi-concavity condition:}
    $u\in C^1$ is quasi-concave $\Longleftrightarrow$ If $u(\vox) \ge u(\voy)$, then $\nabla u(\voy)^\top (\vox - \voy) \ge 0$.

    \begin{proof}
        ($\Rightarrow$) Suppose $u$ is quasi-concave. Then every super-level set 
    $S_\gamma = \{\vox\in\Omega| u(\vox)\ge \gamma\}$ 
    is convex. According to Lemma \ref{lem:1}, for any boundary point 
    $\mathbf{y}$ with $u(\mathbf{y}) = \gamma$, if $\nabla u(\voy) = 0$, the condition holds trivially. Otherwise, $\segu_\gamma$ is contained in the closed half space of $T_\voy$, i.e. 
    \begin{equation}
        \begin{aligned}
        &u(\vox) \ge \gamma = u(\voy) \Rightarrow \vox \in \segu_\gamma \\
        &\Rightarrow \vox\in H_\voy= \{\vox\in\Omega | \nabla u(\voy)^\top (\vox - \voy) \ge 0\}.
    \end{aligned}
    \end{equation}
    
    Since $\gamma$ is arbitrary, for any $\vox,\voy$ with $u(\vox) \ge u(\voy)$, we have 
    $\nabla u(\voy)^\top (\vox - \voy) \ge 0$.

    \medskip
    ($\Leftarrow$)  
    Conversely, fix any $\gamma$ and consider the super-level set $\segu_\gamma$.  
    Take arbitrary $\vox_0,\vox_1\in S_\gamma$ and define 
    \begin{align}
       \phi(t)=u((1-t)\vox_0+t\vox_1), \ t\in[0,1].
    \end{align}
    
 If $\phi$ attained a value strictly below $\min\{\phi(0),\phi(1)\}$ at an interior $t_0\in(0,1)$,  w.l.o.g, let $\phi(1)\ge\phi(0)>\phi(t_0)$. According to continuity and Intermediate Value Theorem, there exists an interval  $[t_0, t_1] \subset[0,1]$ such that  $\phi(t_0) < \phi(t_1) = \phi(0)$, and $\phi(t) \le \phi(0)$ when $t\in [t_0, t_1]$. Such an interval exists by continuity, and can be found easily by extending $(0,\phi(0))$ along $t$ axis and check the intersections with $\phi(t)$. 
 
 Take $t\in (t_0,t_1)$, and  let $\vox_t=(1-t)\vox_0+t\vox_1$.  By the hypothesis, since both $u(\vox_0),u(\vox_1) \ge u(\vox_t)$,
 \begin{align}
     \nabla u(\vox_t)^\top(\vox_i-\vox_t)\ge0, \ i = 0,1,
 \end{align}
 which gives
 \begin{equation}
     \begin{aligned}
         t \nabla u(\vox_t)^\top(\vox_0-\vox_1) & \ge0, \\
         (1-t) \nabla u(\vox_t)^\top(\vox_0-\vox_1) & \le0,
     \end{aligned}
 \end{equation}
collectively,
 \begin{align}
     \nabla u(\vox_t)^\top(\vox_1-\vox_0)=0, \ \forall t\in(t_0,t_1).
 \end{align}

 According to Mean Value Theorem, there exists $t^* \in (t_0,t_1)$, such that 
 \begin{equation}
      \begin{aligned}
     0 &< \phi(t_1) - \phi(t_0) = \nabla u(\vox_{t^*})^\top(\vox_{t_1}-\vox_{t_0}) \\
     & = (t_1 - t_0) \nabla u(\vox_{t^*})^\top (\vox_1-\vox_0)=0,
 \end{aligned}
 \end{equation}
 which is contradictory. Therefore $\phi(t)\ge\min\{\phi(0),\phi(1)\}$ for all $t$, so the segment $[\vox_0,\vox_1]\subset S_\gamma$. Hence $S_\gamma$ is convex. Since $\gamma$ was arbitrary, $u$ is quasi-concave.
    \end{proof}

\section{Proof of Theorem~\ref{thm:3}}
\begin{figure}[htbp]
    \centering
    \includegraphics[width=0.7\linewidth]{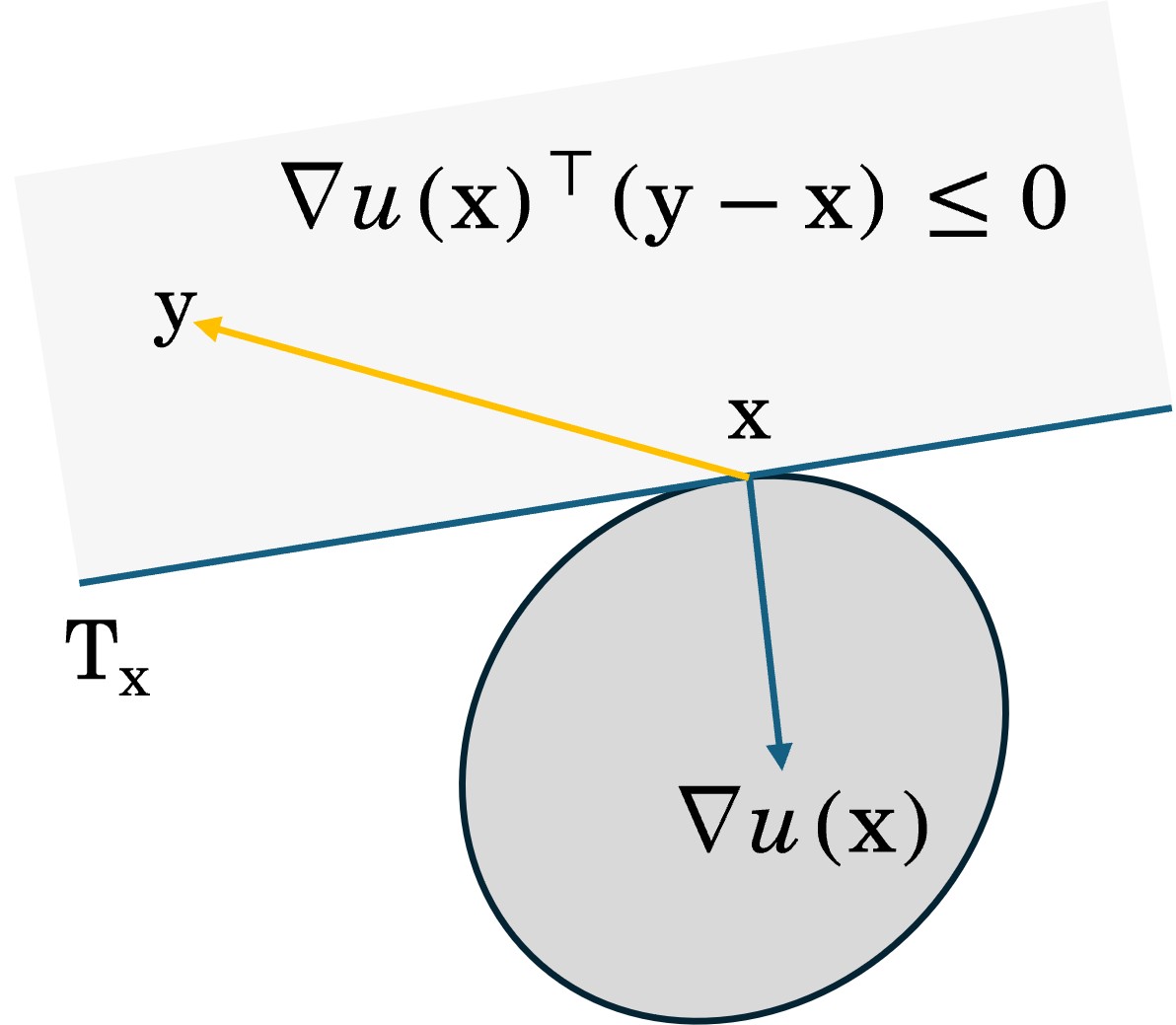}
    \caption{Illustration of the second-order condition. For fixed $\vox$, if $u$ is quasi-concave, then on the half space $\nabla u(\vox)^\top(\voy-\vox)\le 0$, maximum value of $u$ is taken at $\vox$.}
    \label{fig:l2}
\end{figure}

\noindent\textbf{Second-order quasi-concavity sufficient condition:}
Let $u \in C^2$, if the Hessian $\nabla^2 u(\vox) \prec 0$ (strict negative definite) on tangent space $T_\vox$ for all $\vox\in\Omega$ such that $\nabla u(\vox)\neq 0$, then $u \in C^2$ is quasi-concave.
\begin{proof}
    Fix any $\gamma$ and consider the super-level set $\segu_\gamma$.  
    Take arbitrary $\vox_0,\vox_1\in S_\gamma$ and define 
    \begin{equation}
        \begin{aligned}
       \phi(t) &=u((1-t)\vox_0+t\vox_1), \ t\in[0,1], \\
        \vox_t &= (1-t)\vox_0+t\vox_1.
    \end{aligned}
    \end{equation}
  
 If $\phi$ attained a value strictly below $\min\{\phi(0),\phi(1)\}$ at an interior $t_0\in(0,1)$, without loss of generality, let $\phi(1)\ge\phi(0)>\phi(t_0)$. Then according to continuity and Extreme Value Theorem, $\phi$ attains a strict interior minimum at some $t^*\in(0,1)$ such that $\phi(1)\ge\phi(0)>\phi(t^*)$. At this $t^*$, we have $\phi^\prime(t^*) = 0, \phi^{\prime\prime}(t^*) \ge 0$. Further, $\phi^\prime(t^*) = 0$ implies 
 \begin{align}
     \phi^\prime(t^*) = \nabla u(\vox_{t^*})^\top (\vox_1 - \vox_0) = 0,
 \end{align}
 which means $\vox_1 - \vox_0 \in T_{\vox_{t^*}}$. According to the negative definite hypothesis,
 \begin{align}
     (\vox_1 - \vox_0)^\top \nabla^2 u(\vox_{t^*}) (\vox_1 - \vox_0) < 0. 
 \end{align}
 
 However, this contradicts with 
 \begin{align}
     \phi^{\prime\prime}(t^*) = (\vox_1 - \vox_0)^\top \nabla^2 u(\vox_{t^*}) (\vox_1 - \vox_0) \ge 0.
 \end{align}

 Therefore, for all $t$, $\phi(t) \ge \min\{\phi(0), \phi(1)\}$, so the segment $[\vox_0,\vox_1]\subset \segu_\gamma$. Hence $S_\gamma$ is convex. Since $\gamma$ was arbitrary, $u$ is quasi-concave.
\end{proof}

\section{Explicit form of convex loss gradient}

In the proposed convex gradient projection module, the term $\nabla_u \mathcal{L}_{convex}(u)$ can be either computed by auto-gradient, or by using an explicit form, and both approaches are efficient. Let $D_x, D_y, D_{xx}, D_{yy}, D_{xy}$ be fixed linear convolution operators with adjoints $D_x^{\!\top},D_y^{\!\top},D_{xx}^{\!\top},
D_{yy}^{\!\top},D_{xy}^{\!\top}$
realizing first and second order finite differences. Specifically, 
\begin{equation}
    \begin{aligned}
  D_x &= 
\begin{bmatrix}
0 & 0 & 0 \\
0 & -1 & 0 \\
0 & 1 & 0
\end{bmatrix}, 
\ 
D_{xx} = D_x^\top D_x = \begin{bmatrix}
0 & 1 & 0 \\
0 & -2 & 0 \\
0 & 1 & 0
\end{bmatrix}, \\
D_y &= 
\begin{bmatrix}
0 & 0 & 0 \\
0 & -1 & 1 \\
0 & 0 & 0
\end{bmatrix}, 
D_{yy} = D_y^\top D_y = \begin{bmatrix}
0 & 0 & 0 \\
1 & -2 & 1 \\
0 & 0 & 0
\end{bmatrix}, \\
D_{xy} &= D_{yx} = \frac{1}{2}(D_y D_x + D_x D_y) = \frac{1}{2}\begin{bmatrix}
0 & 0 & 0 \\
0 & -1 & 1 \\
0 & 1 & -1
\end{bmatrix}.
\end{aligned}
\end{equation}

For any function $u:\Omega\to[0,1]$ define
\begin{equation}
    \begin{aligned}
        u_x&=D_xu,\quad u_y=D_yu, \\u_{xx}&=D_{xx}u,\quad u_{yy}=D_{yy}u,\quad u_{xy}=D_{xy}u.
    \end{aligned}
\end{equation}

\paragraph{Explicit gradient of $\mathcal{L}_{2nd}.$} The loss is
\begin{align}
    \mathcal{L}_{2nd}(u)=\frac{1}{|\Omega|}\sum_{\mathbf{x}\in\Omega}\|\nabla u(\mathbf{x})\|R(\mathbf{x}),
\end{align}
where  $R = \mathrm{ReLU}(Q_2+\delta)$, and $H = \mathbf{1}[Q_2+\delta>0]$ (indicator function). Therefore,
\begin{equation}
    \begin{aligned}
    \nabla_u\mathcal{L}_{2nd}(u)&=\frac{1}{|\Omega|}\sum_{\mathbf{x}\in\Omega}[\nabla_u\|\nabla u(\mathbf{x})\|R(\mathbf{x}) \\&+ \|\nabla u(\mathbf{x})\|\nabla_uR(\mathbf{x})].
\end{aligned}
\end{equation}

With $Q_2 = u_x^2u_{yy}-2u_xu_yu_{xy}+u_y^2u_{xx}$,
the pointwise partial derivatives are
\begin{equation}
    \begin{aligned}
        \frac{\partial Q_2}{\partial u_x}&=2u_xu_{yy}-2u_yu_{xy}, \
\frac{\partial Q_2}{\partial u_y}=2u_yu_{xx}-2u_xu_{xy}, \\
\frac{\partial Q_2}{\partial u_{xx}}&=u_y^2, \
\frac{\partial Q_2}{\partial u_{yy}}=u_x^2, \
\frac{\partial Q_2}{\partial u_{xy}}=-2u_xu_y .
\end{aligned}
\end{equation}

Since $u_x=D_xu$ etc.\ are linear in $u$, the chain rule with operator adjoints
gives the discrete gradient
\begin{equation}
    \begin{aligned}
        &\nabla_u Q_2(u) = D_{xx}^{\top}(u_y^2)
+ D_{yy}^{\top}(u_x^2)
+ D_{xy}^{\top}(-2u_xu_y)
\\ & + D_x^{\top}(2u_xu_{yy}-2u_yu_{xy}) + D_y^{\top}(2u_yu_{xx}-2u_xu_{xy}).
\end{aligned}
\end{equation}

For numerical stability use the smoothed magnitude
\begin{equation}
    \|\nabla u\| = \sqrt{u_x^2+u_y^2+\varepsilon}\quad (\varepsilon>0),
\end{equation}
then
\begin{equation}
    \frac{\partial \|\nabla u\|}{\partial u_x}=\frac{u_x}{\|\nabla u\|},
\quad
\frac{\partial \|\nabla u\|}{\partial u_y}=\frac{u_y}{\|\nabla u\|},
\end{equation}
and the chain rule yields
\begin{equation}
\nabla_u \|\nabla u\| =
D_x^{\!\top}\!\left(\frac{u_x}{\|\nabla u\|}\right)
\;+\;
D_y^{\!\top}\!\left(\frac{u_y}{\|\nabla u\|}\right).    
\end{equation}

Substituting above with $\nabla_u R = H\nabla_u Q_2$ gives
\begin{equation}
\begin{aligned}
&\nabla_u \mathcal{L}_{2\mathrm{nd}}(u)
=\frac{1}{|\Omega|}\Big[
\underbrace{D_x^{\!\top}\!\Big(R\,\frac{u_x}{\|\nabla u\|}\Big) + D_y^{\!\top}\!\Big(R\,\frac{u_y}{\|\nabla u\|}\Big)}_{\text{through }g=\|\nabla u\|_\varepsilon}
\\
&\quad+\underbrace{D_x^{\!\top}\!\Big(\|\nabla u\|H\,(2u_xu_{yy}-2u_yu_{xy})\Big)}_{\text{via } \partial Q_2/\partial u_x}
\\
&\quad+\underbrace{D_y^{\!\top}\!\Big(\|\nabla u\|H\,(2u_yu_{xx}-2u_xu_{xy})\Big)}_{\text{via }\ \partial Q_2/\partial u_y}
\\
&\quad+\underbrace{D_{xx}^{\!\top}\!\big(\|\nabla u\|H\,u_y^2\big)
+ D_{yy}^{\!\top}\!\big(\|\nabla u\|H\,u_x^2\big)}_{\text{via } \partial Q_2/\partial u_{xx},\partial Q_2/\partial u_{yy}}
\Big]
\\
&\quad+\underbrace{D_{xy}^{\!\top}\!\big(\|\nabla u\|H\,(-2u_xu_y)\big)}_{\text{via }\partial Q_2/\partial u_{xy}}
\Big].
\end{aligned}    
\end{equation}

With the elementwise sigmoid $u =\sigma(o)=1/(1+e^{-o})$ one has
$\tfrac{\partial u}{\partial o}=u\,(1-u)$, hence
\[
\nabla_o \mathcal{L}_{2\mathrm{nd}}(u(o))
\;=\;
\big(\nabla_u \mathcal{L}_{2\mathrm{nd}}(u)\big)\;\odot\; u\,(1-u).
\]

\paragraph{Explicit gradient of $\mathcal{L}_{1st}.$}
Let $D_x, D_y$ be fixed linear convolution operators and write
\begin{equation}
\nabla u(\voy)=\big(D_x u(\voy),\,D_y u(\voy)\big)^\top.    
\end{equation}

For a temperature $\varepsilon>0$ define the smoothed sigmoid
\begin{equation}
\sigma_\varepsilon(t)\;=\;\frac{1}{1+e^{-t/\varepsilon}}, \
\sigma_\varepsilon'(t)\;=\;\frac{1}{\varepsilon}\,\sigma_\varepsilon(t)\bigl(1-\sigma_\varepsilon(t)\bigr).    
\end{equation}

Given a neighborhood $\mathcal{N}_{\voy}\subset\Omega$ around each $\voy$, set
for every ordered pair $(\vox,\voy)$ with $\vox\in\mathcal{N}_{\voy}$:
\begin{equation}
\begin{aligned}
&S(\vox,\voy)\;=\;\sigma_\varepsilon\big(u(\vox)-u(\voy)\big),\quad
\mathbf{v}_{\vox,\voy}\;=\;\vox-\voy, \\
&R(\vox,\voy)=H(\vox,\voy)\,\bigl(-\nabla u(\voy)^\top \mathbf{v}_{\vox,\voy}\bigr),
\end{aligned}
\end{equation}
where $H(\vox,\voy)=\mathbf{1}[\nabla u(\voy)^\top \mathbf{v}_{\vox,\voy}<0]$. The loss reads
\[
\mathcal{L}_{1\mathrm{st}}(u)
=\frac{1}{|\Omega|}\sum_{\voy\in\Omega}\;\sum_{\vox\in\mathcal{N}_{\voy}}
S(\vox,\voy)\,R(\vox,\voy).
\]

For any pixel $p\in\Omega$, the partial derivative $\big[\nabla_u \mathcal{L}_{1st}(u)\big](p)$
collects all terms in which $u(p)$ appears. There are three kinds of appearances.

\noindent\textbf{(A) $u(p)$ as a ``source'' value $u(\vox)$ in $S(\vox,\voy)$.}
Here $u(\vox)$ contributes when $\vox=p$ and $p\in\mathcal{N}_{\voy}$:
\begin{equation}
\frac{\partial S(\vox,\voy)}{\partial u(p)}=
\begin{cases}
\sigma_\varepsilon'\!\big(u(p)-u(\voy)\big), & \vox=p,\\[2pt]
0, & \vox\neq p.
\end{cases}    
\end{equation}

Thus the contribution is
\begin{equation}
\sum_{\voy:\;p\in\mathcal{N}_{\voy}}\sigma_\varepsilon'\!\big(u(p)-u(\voy)\big)\,R(p,\voy). 
\end{equation}

\noindent\textbf{(B) $u(p)$ as an ``anchor'' value $u(\voy)$ in $S(\vox,\voy)$.}
Here $u(\voy)$ contributes when $\voy=p$ and $\vox\in\mathcal{N}_{p}$:
\begin{equation}
\frac{\partial S(\vox,\voy)}{\partial u(p)}=
\begin{cases}
-\sigma_\varepsilon'\!\big(u(\vox)-u(p)\big), & \voy=p,\\[2pt]
0, & \voy\neq p.
\end{cases}    
\end{equation}

Thus the contribution is
\begin{equation}
-\sum_{\vox\in\mathcal{N}_{p}}\sigma_\varepsilon'\!\big(u(\vox)-u(p)\big)\,R(\vox,p).    
\end{equation}

\noindent\textbf{(C) $u(p)$ through the anchor gradient $\nabla u(\voy)$ inside $R(\vox,\voy)$.}
$R(\vox,\voy)$ depends on $u$ through the anchor gradient $\nabla u(\voy)$. Collecting all such contributions over $\voy$, the resulting gradient can be written compactly using the adjoint operators $D^\top_x,D^\top_y$.

Write $v_x(\vox,\voy)=(\vox-\voy)_x$, $v_y(\vox,\voy)=(\vox-\voy)_y$. Then
\begin{equation}
\begin{aligned}
R(\vox,\voy) =-H(\vox,\voy)\Bigl(D_xu(\voy)\,v_x(\vox,\voy)+D_yu(\voy)\,v_y(\vox,\voy)\Bigr).
\end{aligned}
\end{equation}

Introduce two auxiliary fields $C_x,C_y:\Omega\to\mathbb{R}$,
\begin{equation}
\begin{aligned}
C_x(p)= -\!\sum_{\vox\in\mathcal{N}_{p}} S(\vox,p)\,H(\vox,p)\,v_x(\vox,p),\\
C_y(p)= -\!\sum_{\vox\in\mathcal{N}_{p}} S(\vox,p)\,H(\vox,p)\,v_y(\vox,p),
\end{aligned}
\end{equation}

Therefore, the part of the objective depending on $R$ can be rewritten as
\begin{equation}
\begin{aligned}
 &\sum_{\voy\in\Omega}\sum_{\vox\in N_{\voy}} S(\vox,\voy)R(\vox,\voy)
\\
&=
\sum_{\voy\in\Omega} C_x(\voy)\,D_xu(\voy)
+
\sum_{\voy\in\Omega} C_y(\voy)\,D_yu(\voy).
\end{aligned}
\end{equation}

To differentiate it, note that both $D_x,D_y$ are linear operators, and we use the following standard identity: for any linear operator
$A$ and any coefficient field $c$,
\begin{equation}
\begin{aligned}
\sum_{\voy} c(\voy)\,(Au)(\voy)
=
c^\top Au
=
(A^\top c)^\top u,
\end{aligned}
\end{equation}

hence
\begin{equation}
\begin{aligned}
\nabla_u \left(\sum_{\voy} c(\voy)\,(Au)(\voy)\right)=A^\top c.    
\end{aligned}
\end{equation}

Applying this identity to the two terms gives
\begin{equation}
\begin{aligned}
&\left(
\nabla_u
\sum_{\voy\in\Omega}\sum_{\vox\in N_{\voy}}
S(\vox,\voy)R(\vox,\voy)
\right)_{\mathrm{via}\ R}
\\
&=
D_x^\top C_x + D_y^\top C_y.    
\end{aligned}
\end{equation}

Summing (A)+(B)+(C),
\begin{equation}
\begin{aligned}
\big[\nabla_u \mathcal{L}_{1\mathrm{st}}(u)\big](p)
&= \frac{1}{|\Omega|}
\Big[\underbrace{\sum_{\voy:\;p\in\mathcal{N}_{\voy}}
\sigma_\varepsilon'\!\big(u(p)-u(\voy)\big)\,R(p,\voy)}_{\text{via } \partial S/\partial u}
 \\
&-\underbrace{\sum_{\vox\in\mathcal{N}_{p}}
\sigma_\varepsilon'\!\big(u(\vox)-u(p)\big)\,R(\vox,p)}_{\text{via } \partial S/\partial u}
\\[2pt]
&+\underbrace{\big[D_x^{\!\top}C_x + D_y^{\!\top}C_y\big](p)}_{\text{via } \partial R/\partial u}
\Big],
\end{aligned}    
\end{equation}

If $u=\sigma(o)$ is the elementwise sigmoid of logits $o$,
then $\nabla_o\mathcal{L}_{1\mathrm{st}}(u(o))=\big(\nabla_u\mathcal{L}_{1\mathrm{st}}(u)\big)\odot u(1-u)$ elementwise.

\clearpage
\section{Qualitative Visualization on  ACDC Dataset.}
\begin{center}
\vspace{3em}
\includegraphics[width=1\textwidth]{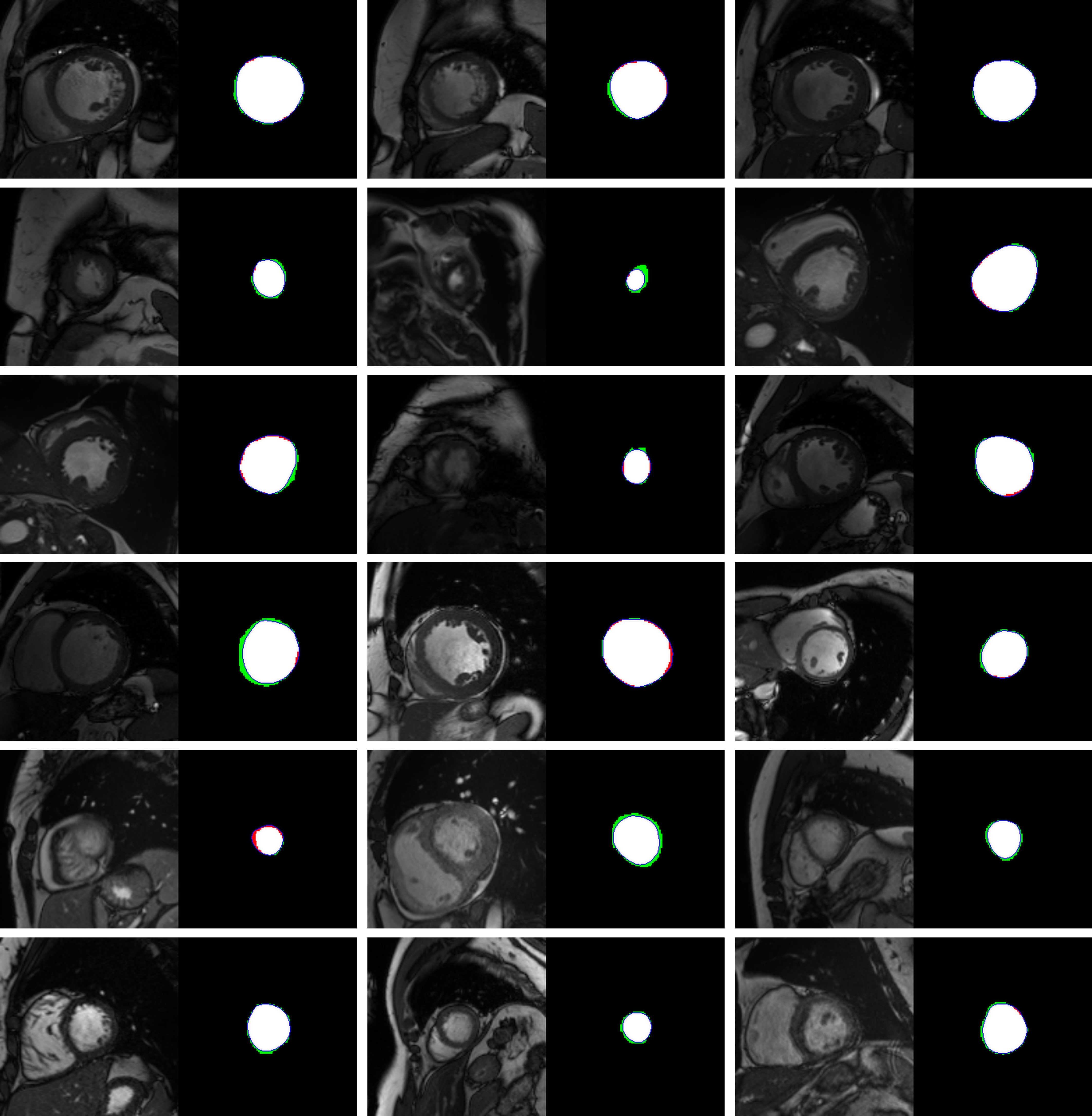}
\end{center}

\clearpage
\section{Qualitative Visualization on CASIA Dataset.}
\begin{center}
\vspace{3em}
\includegraphics[width=1\textwidth]{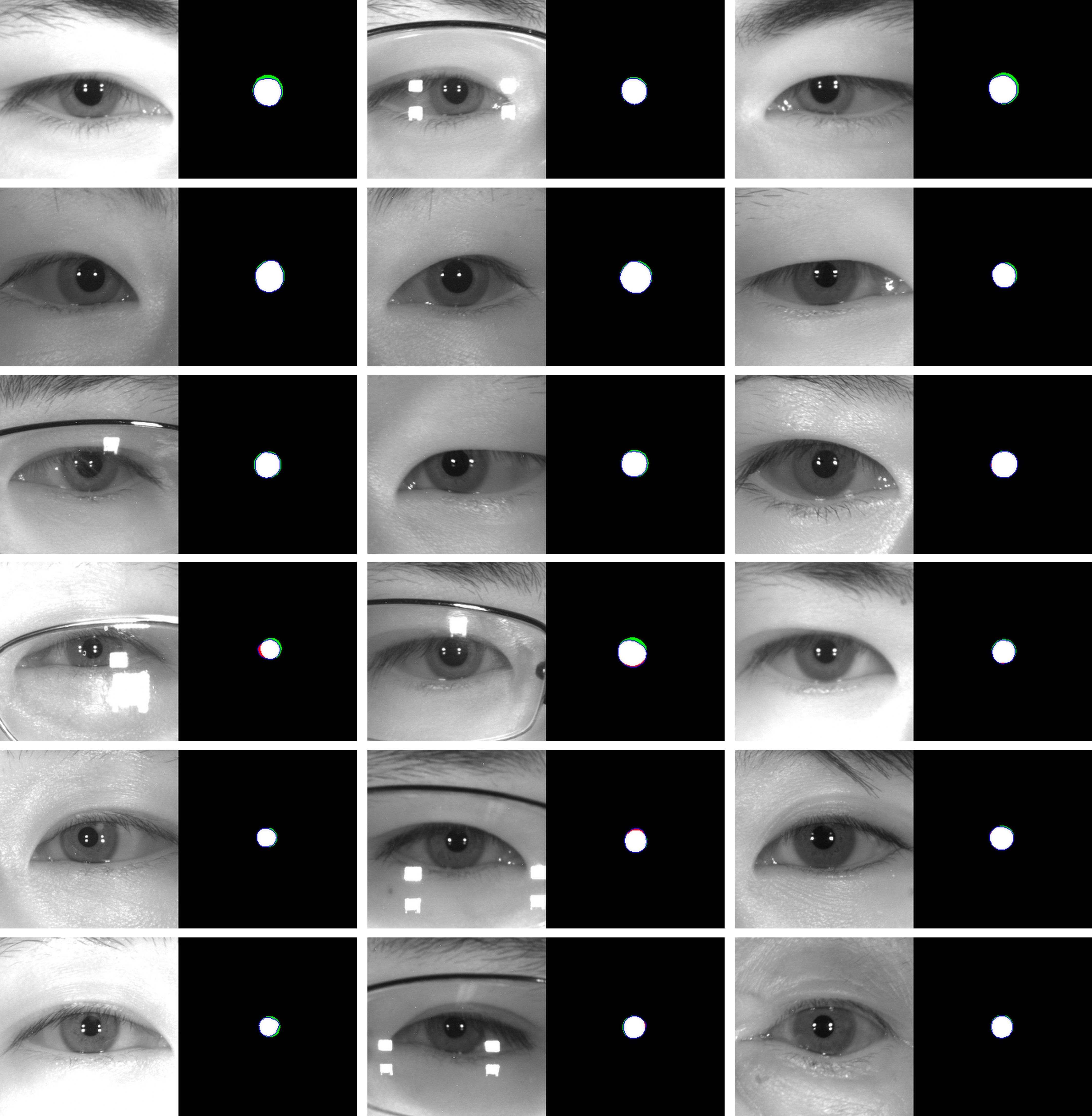}
\end{center}

\clearpage
\section{Qualitative Visualization on REFUGE Dataset.}
\begin{center}
\vspace{3em}
\includegraphics[width=1\textwidth]{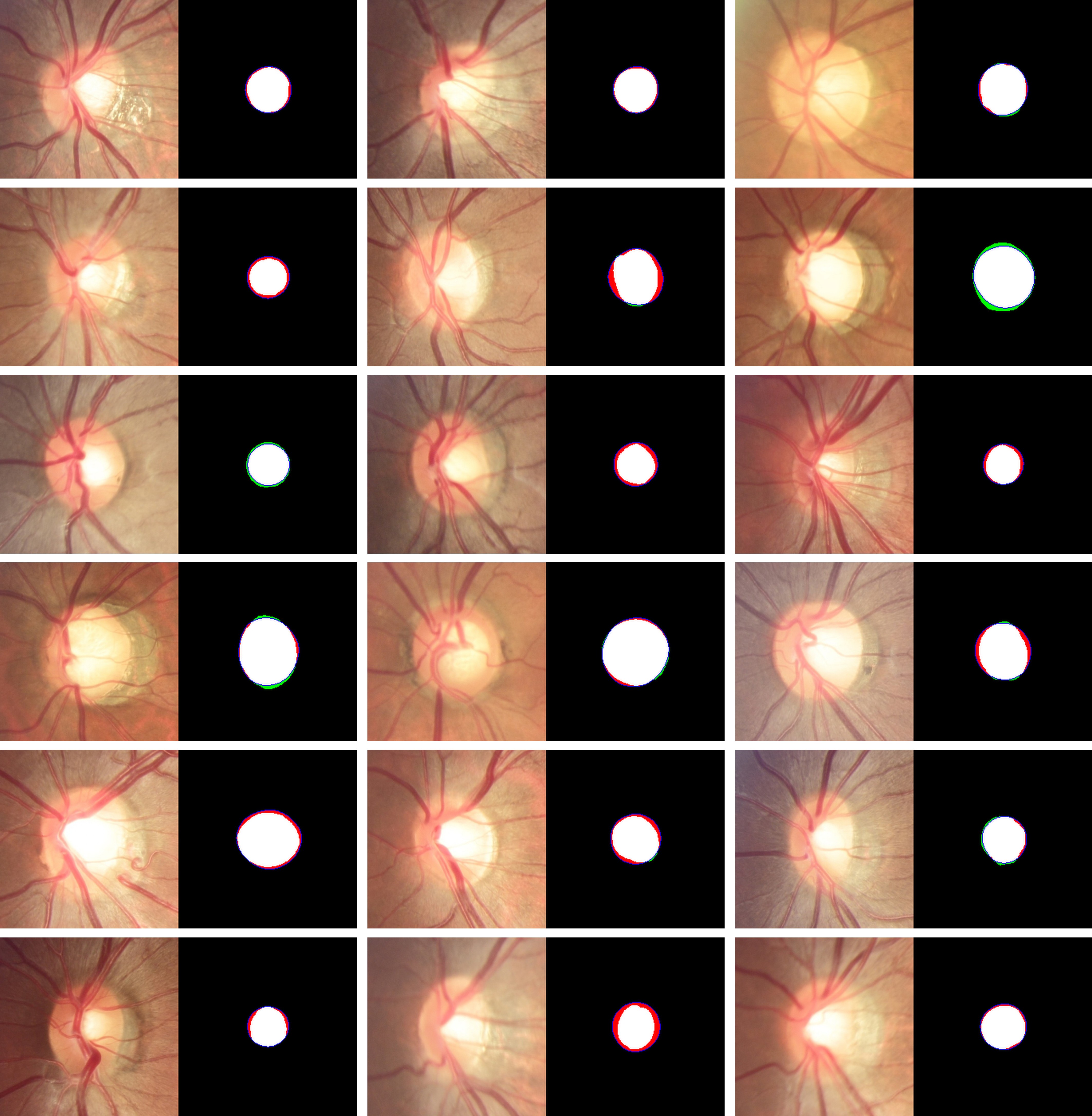}
\end{center}

\clearpage
\section{Qualitative Visualization on RIM-ONE-r3 Dataset.}
\begin{center}
\vspace{3em}
\includegraphics[width=1\textwidth]{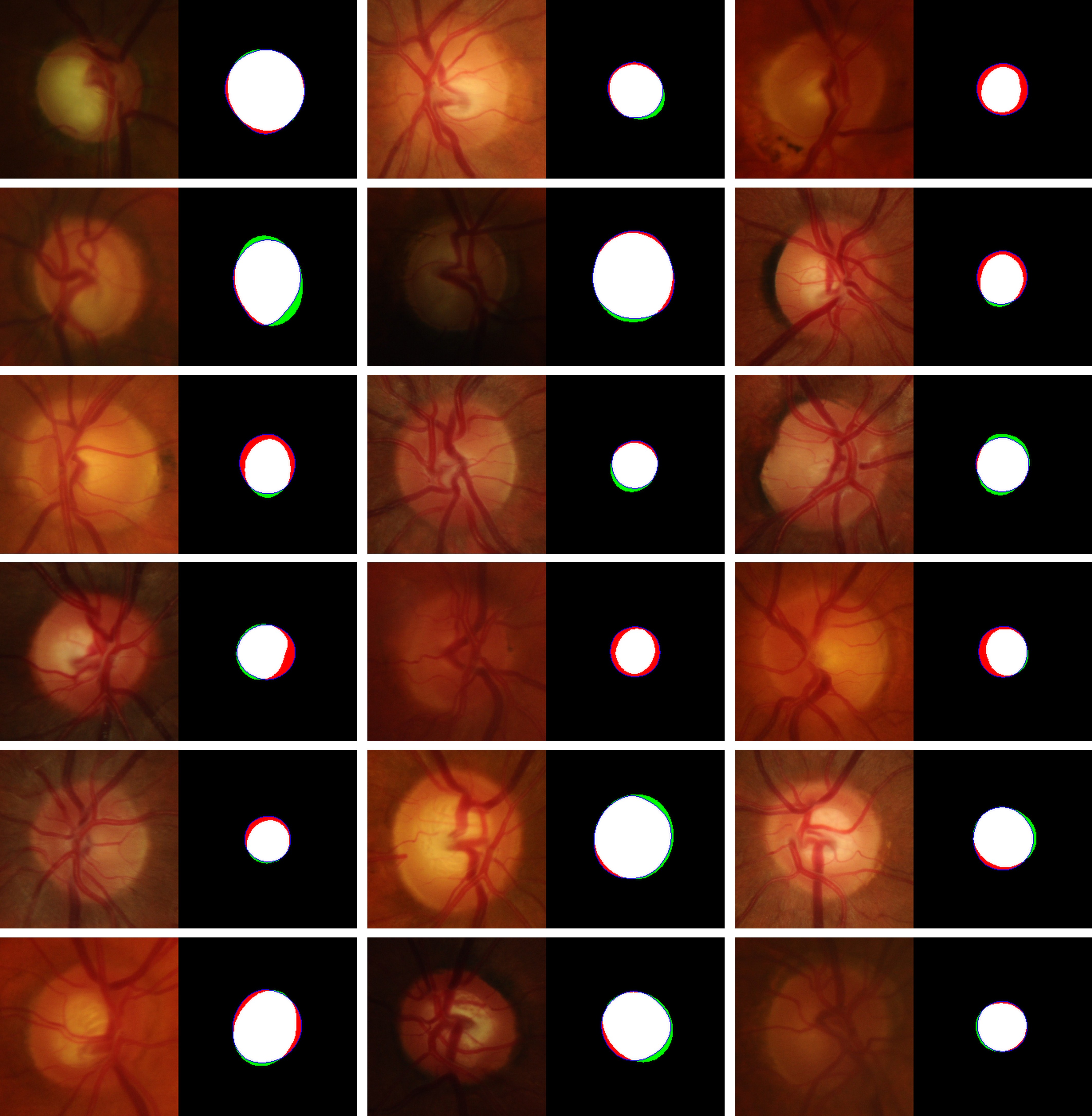}
\end{center}

\clearpage
\section{Ten Trials Results for the Ablation Study}
\begin{center}
\resizebox{1\textwidth}{!}{
\begin{tabular}{c c c c c c c c c c c c c c}
\toprule
\textbf{Trial}  & \textbf{Prior} &  \textbf{Backbone} &\textbf{Dice} & \textbf{IoU} & \textbf{HD} &  \textbf{Backbone} &\textbf{Dice} & \textbf{IoU} & \textbf{HD} &  \textbf{Backbone} &\textbf{Dice} & \textbf{IoU} & \textbf{HD} \\
\midrule
\multirow{4}{*}{0} &  -  & \multirow{4}{*}{U-Net} & 83.02 & 70.97 & 12.56 & \multirow{4}{*}{Swin-Unet} & 85.32 & 74.40 & 7.530 & \multirow{4}{*}{DeepLabV3p} & 74.52 & 59.39 & 12.18 \\
& $0th$ & & 82.93 & 70.83 & 12.52 & & 85.19 & 74.20 & 7.516 & & 74.49 & 59.35 & 12.17       \\
& $1st$ & & 83.59 & 71.81 & 11.67  & & 86.13 & 75.64 & 7.093 & & 76.66 & 62.15 & 11.09   \\
& $2nd$ & & 84.78 & 73.58 & 9.837 & & 88.53 & 79.43 & 5.969 & & 82.12 & 69.67 & 8.527    \\
\midrule
\multirow{4}{*}{1} &  -  & \multirow{4}{*}{U-Net} & 84.28 & 72.83 & 11.04 & \multirow{4}{*}{Swin-Unet} & 86.65 & 76.44 & 6.965 & \multirow{4}{*}{DeepLabV3p} & 77.21 & 62.88 & 11.04 \\
& $0th$ & & 84.22 & 72.74 & 10.99 & & 86.56 & 76.31 & 6.942 & & 77.16 & 62.81 & 11.04       \\
& $1st$ & & 84.88 & 73.74 & 10.31  & & 87.23 & 77.35 & 6.464 & & 79.02 & 65.31 & 10.10   \\
& $2nd$ & & 86.30 & 75.90 & 8.764 & & 88.74 & 79.75 & 5.665 & & 82.93 & 70.84 & 8.244    \\
\midrule
\multirow{4}{*}{2} &  -  & \multirow{4}{*}{U-Net} & 83.66 & 71.91 & 11.99 & \multirow{4}{*}{Swin-Unet} & 88.72 & 79.73 & 6.137 & \multirow{4}{*}{DeepLabV3p} & 75.97 & 61.25 & 11.42 \\
& $0th$ & & 83.31 & 71.39 & 11.86 & & 88.69 & 79.67 & 6.117 & & 75.95 & 61.23 & 11.42       \\
& $1st$ & & 83.83 & 72.16 & 11.02  & & 89.31 & 80.69 & 5.783 & & 77.96 & 63.88 & 10.42   \\
& $2nd$ & & 84.85 & 73.69 & 9.687 & & 90.73 & 83.03 & 5.043 & & 82.80 & 70.65 & 8.183    \\
\midrule
\multirow{4}{*}{3} &  -  & \multirow{4}{*}{U-Net} & 85.48 & 74.65 & 9.804 & \multirow{4}{*}{Swin-Unet} & 86.29 & 75.89 & 8.377 & \multirow{4}{*}{DeepLabV3p} & 77.54 & 63.32 & 10.74 \\
& $0th$ & & 85.50 & 74.68 & 9.640 & & 86.14 & 75.65 & 8.256 & & 77.53 & 63.30 & 10.73       \\
& $1st$ & & 85.97 & 75.39 & 8.973  & & 87.11 & 77.17 & 7.542 & & 79.89 & 66.51 & 9.618   \\
& $2nd$ & & 86.09 & 75.58 & 8.055 & & 88.79 & 79.84 & 7.082 & & 84.66 & 73.41 & 7.390    \\
\midrule
\multirow{4}{*}{4} &  -  & \multirow{4}{*}{U-Net} & 84.98 & 73.88 & 10.20 & \multirow{4}{*}{Swin-Unet} & 87.01 & 77.01 & 10.59 & \multirow{4}{*}{DeepLabV3p} & 77.20 & 62.87 & 11.26 \\
& $0th$ & & 84.93 & 73.80 & 10.16 & & 86.81 & 76.69 & 10.50 & & 77.19 & 62.85 & 11.25       \\
& $1st$ & & 85.56 & 74.76 & 9.380  & & 87.44 & 77.68 & 8.022 & & 79.17 & 65.52 & 10.31   \\
& $2nd$ & & 86.50 & 76.21 & 8.106 & & 87.95 & 78.49 & 7.349 & & 83.80 & 72.12 & 8.092    \\
\midrule
\multirow{4}{*}{5} &  -  & \multirow{4}{*}{U-Net} & 83.74 & 72.02 & 11.86 & \multirow{4}{*}{Swin-Unet} & 84.35 & 72.94 & 7.777 & \multirow{4}{*}{DeepLabV3p} & 78.40 & 64.48 & 10.93 \\
& $0th$ & & 83.57 & 71.78 & 11.78 & & 84.26 & 72.81 & 7.777 & & 78.36 & 64.41 & 10.93       \\
& $1st$ & & 84.25 & 72.79 & 11.02  & & 85.13 & 74.11 & 7.330 & & 80.38 & 67.20 & 9.616   \\
& $2nd$ & & 85.66 & 74.92 & 8.977 & & 87.51 & 77.80 & 6.243 & & 84.25 & 72.78 & 7.580    \\
\midrule
\multirow{4}{*}{6} &  -  & \multirow{4}{*}{U-Net} & 83.95 & 72.34 & 11.49 & \multirow{4}{*}{Swin-Unet} & 83.07 & 71.05 & 8.197 & \multirow{4}{*}{DeepLabV3p} & 76.80 & 62.34 & 12.20 \\
& $0th$ & & 83.68 & 71.94 & 11.40 & & 82.99 & 70.92 & 8.190 & & 76.81 & 62.35 & 12.20       \\
& $1st$ & & 84.17 & 72.67 & 10.68  & & 83.80 & 72.12 & 7.773 & & 79.62 & 66.14 & 10.18   \\
& $2nd$ & & 85.03 & 73.96 & 9.489 & & 85.90 & 75.28 & 6.847 & & 85.43 & 74.57 & 7.370    \\
\midrule
\multirow{4}{*}{7} &  -  & \multirow{4}{*}{U-Net} & 83.68 & 71.94 & 11.69 & \multirow{4}{*}{Swin-Unet} & 84.58 & 73.28 & 8.912 & \multirow{4}{*}{DeepLabV3p} & 77.72 & 63.55 & 10.72 \\
& $0th$ & & 83.44 & 71.59 & 11.58 & & 84.11 & 72.58 & 8.783 & & 77.69 & 63.52 & 10.71       \\
& $1st$ & & 84.01 & 72.43 & 10.91  & & 85.38 & 74.48 & 7.938 & & 79.99 & 66.66 & 9.525   \\
& $2nd$ & & 85.33 & 74.41 & 9.558 & & 87.64 & 77.99 & 6.789 & & 83.95 & 72.34 & 7.462    \\
\midrule
\multirow{4}{*}{8} &  -  & \multirow{4}{*}{U-Net} & 85.03 & 73.96 & 10.39 & \multirow{4}{*}{Swin-Unet} & 86.99 & 76.98 & 6.555 & \multirow{4}{*}{DeepLabV3p} & 78.96 & 65.24 & 10.39 \\
& $0th$ & & 84.91 & 73.77 & 10.30 & & 86.93 & 76.89 & 6.539 & & 78.96 & 65.23 & 10.38       \\
& $1st$ & & 85.52 & 74.70 & 9.608  & & 87.63 & 77.98 & 6.169 & & 80.39 & 67.21 & 9.661   \\
& $2nd$ & & 86.85 & 76.76 & 8.281 & & 89.27 & 80.62 & 5.431 & & 83.97 & 72.37 & 7.832    \\
\midrule
\multirow{4}{*}{9} &  -  & \multirow{4}{*}{U-Net} & 83.50 & 71.67 & 11.96 & \multirow{4}{*}{Swin-Unet} & 85.82 & 75.16 & 8.329 & \multirow{4}{*}{DeepLabV3p} & 77.73 & 63.57 & 11.19 \\
& $0th$ & & 83.27 & 71.33 & 11.77 & & 85.27 & 74.32 & 8.193 & & 77.69 & 63.51 & 11.18       \\
& $1st$ & & 83.83 & 72.16 & 11.09  & & 86.23 & 75.79 & 7.387 & & 79.54 & 66.03 & 10.22   \\
& $2nd$ & & 85.03 & 73.96 & 9.832 & & 87.49 & 77.76 & 6.648 & & 83.83 & 72.16 & 8.148    \\
\bottomrule
\end{tabular}
}
\end{center}


\end{document}